\documentclass[10pt]{article}
\usepackage[margin=1in,columnsep=0.25in]{geometry}
\usepackage{amsmath,amssymb,amsthm}
\usepackage{graphicx}
\usepackage{booktabs}
\usepackage{multirow}
\usepackage{subcaption}
\usepackage{xcolor}
\usepackage{hyperref}
\usepackage{algorithm}
\usepackage{algpseudocode}
\usepackage{microtype}
\usepackage{natbib}
\usepackage{array}
\usepackage{float}

\definecolor{myblue}{RGB}{0,51,102}
\hypersetup{colorlinks=true,linkcolor=myblue,citecolor=myblue,urlcolor=myblue}
\newtheorem{theorem}{Theorem}

\title{\Large\textbf{HiFA4: Training-Free 4-bit FlashAttention\\
on Ascend HIF4 NPUs for LLM Inference}}

\author{
  \small
  Hui Dong$^{*}$ \quad
  Yanzhao Li \quad  
  Jie Gao \quad  
  Chunlu Li \\[3pt]
  \small
  Zhiyuan Zhang \quad
  Yupeng Sun \quad
  Zhenyuan Chen\quad
  Zhiqiang Zou\\[3pt]
  \normalsize Huawei Technologies \\
  \small \texttt{donghui24@huawei.com}
}

\date{}

\begin{document}
\maketitle

\begin{abstract}
We present HiFA4, a post-training operator-level design that executes
both the $QK^T$ and $PV$ matrix multiplications in FlashAttention as
4-bit HIF4 Cube GEMMs for LLM inference on Ascend NPUs, while
maintaining online softmax state in FP16.
To our knowledge, HiFA4 is the first such design evaluated on standard
NLP benchmarks.

We introduce two complementary mechanisms.
\textbf{Smooth-QK} applies a calibration-static, per-channel equivalent
rescaling to Q and K after RoPE, transferring quantization difficulty
from K to Q.
Because the target channels are determined by fixed model parameters
rather than input data, the scale factors are computed once from a small
offline calibration set and require no per-tile online reduction at
inference (only a fusible element-wise scaling).
\textbf{P-Reordering} ensures that the softmax normalizer is accumulated
from the same quantized attention weights $\hat{P}$ used in the $PV$
GEMM, rather than a higher-precision reconstruction.
We prove (Theorem~\ref{thm:main}) that the conventional inconsistent
formulation introduces a systematic, coherent output-scaling error, and
validate empirically on a Qwen3-8B Layer-0 MMLU trace that this error is
present in 100\% of 3.6~million real attention tiles with median
magnitude $\bar\varepsilon=-0.064$.
P-Reordering additionally allows the normalizer accumulation to be fused
into the $PV$ Cube GEMM.

Evaluated on five LLMs, HiFA4 consistently reduces quantization-induced
decision drift.
On Qwen3-8B, HiFA4 recovers 37.5\% of the accuracy gap introduced by
direct HIF4 quantization (sample-weighted $\delta_w$ narrowing from
1.12\,pp to 0.70\,pp, Table~\ref{tab:calib}),
reduces the fraction of BF16-inconsistent predictions on MMLU from 16.3\%
to 8.2\%, and cuts accuracy regressions (samples correctly answered by
BF16 but incorrectly by the quantized model) by 57\%
($1{,}071\rightarrow465$ on MMLU full set).
On Gemma2-9B, mild smoothing keeps HiFA4 within 0.7~pp of BF16 while
cutting MMLU accuracy regressions by 27\% relative to direct quantization.
On the three models where the applicability gate disables Smooth-QK
(LLaMA3.1-8B, Mistral-7B, Phi-4B), the remaining HiFA4 components---
P-Reordering together with the adopted Q-Mean auxiliary---reduce full-set
MMLU accuracy regressions by 41--52\% (and HellaSwag by 39--42\%) with no
K-smoothing, showing that
this part of the design generalizes beyond models with concentrated
K-outliers.
A theoretical instruction-scheduling analysis projects a 35.4\%
critical-path latency reduction relative to BF16 by fusing the
softmax normalizer into the $PV$ Cube GEMM via P-Reordering;
on-hardware validation will be reported when the target
Ascend NPU becomes publicly available.
\end{abstract}

\section{Introduction}
\label{sec:intro}

FlashAttention~\citep{dao2022fa,dao2024fa2,shah2024fa3,zadouri2026fa4}
is the widely-adopted tiled attention kernel for LLM inference.
Its two dominant matrix multiplications, $QK^T$ and $PV$, constitute
the majority of attention compute in the prefill stage.
Executing both as 4-bit Cube GEMMs using Ascend HIF4 would substantially
improve throughput, since 4-bit operands allow the Cube units to process
four times as many elements per cycle compared to BF16.
The challenge is preserving accuracy: direct 4-bit quantization of
attention activations causes substantial degradation in language models.

Prior FP4 attention work has mainly focused on diffusion and
video-generation models.
SageAttention3~\citep{zhang2025sage3} reports FP4 inference results
on CogVideoX, HunyuanVideo, Mochi, Flux, and SD3.5, but its
NLP-related experiments do not correspond to a full post-training FP4
inference evaluation.
Concurrent work targets NVIDIA NVFP4 rather than Ascend HIF4 and uses
different correction strategies: online scale search with
mixed-precision KV caching~\citep{gupta2026scalesearch}, or selective
mixed precision~\citep{sharratt2026thrift}.

We make two technical contributions:
(1)~\textbf{Smooth-QK}: calibration-static suppression of
QK-RMSNorm-induced K-activation outliers with no per-tile online
reduction (only a fusible element-wise scaling); and
(2)~\textbf{P-Reordering}: proof and elimination of a coherent output
error introduced by normalizer--GEMM inconsistency, plus a Cube-path
efficiency benefit.
Alongside standard accuracy, we report prediction flip rate and
per-sample accuracy regressions relative to BF16, which expose
distributional shifts in decision boundaries that aggregate accuracy
can mask.

\section{Related Work}

\paragraph{FlashAttention.}
FA1--FA4~\citep{dao2022fa,dao2024fa2,shah2024fa3,zadouri2026fa4}
optimize the exact tiled attention kernel.
HiFA4 addresses the numerical accuracy of its GEMMs rather than
kernel scheduling or memory layout.

\paragraph{LLM quantization.}
Weight-only post-training methods (GPTQ~\citep{frantar2022gptq},
AWQ~\citep{lin2024awq}) and fine-tuning-based approaches
(QLoRA~\citep{dettmers2024qlora}) target static model weights.
SmoothQuant~\citep{xiao2023smooth} and QuaRot~\citep{ashkboos2024quarot}
reduce model-level activation or weight outliers.
KIVI~\citep{liu2024kivi} quantizes the KV cache.
HiFA4 operates inside the fused FA kernel on Q, K, V, and $P$.

\paragraph{Quantized attention.}
\citet{zhang2025sage2} introduced Q-Mean compensation for sub-8-bit
attention; we adopt this technique as an auxiliary component.
\citet{zhang2025sage3} demonstrated the first FP4 FA kernel on NVIDIA
Blackwell, evaluated on image and video generation.
\citet{gupta2026scalesearch} and \citet{sharratt2026thrift} address
NVFP4 for causal LM and long-context, respectively, on NVIDIA hardware.
HiFA4 targets Ascend HIF4's three-level hierarchical format and derives
its corrections from a parameter-level root-cause analysis (Smooth-QK)
and an algebraic consistency argument (P-Reordering).

\section{Background}

\paragraph{HIF4.}
The 4-bit hierarchical block-floating-point format from~\citet{huawei2026hif4}
for Ascend NPUs packs 64 elements per unit with 32 bits of shared scaling
metadata (averaging 4.5 bits per value).
The metadata defines a three-level scaling hierarchy: a level-1 global
base scale (E6M2) per 64-element group, plus level-2 (8-way) and level-3
(16-way) 1-bit micro-exponents that refine the intra-group dynamic range;
each 4-bit element is a sign-magnitude S1P2 value (equivalent to E1M2,
3-bit significand).
All QDQ operations follow the official specification; we use the
BF16$\to$HIF4 conversion of Algorithm~1 in the HIF4 specification~\citep{huawei2026hif4} without
modification.
We use ``4-bit'' to refer to the per-element compute operand width;
storage accounting includes the 32-bit per-group metadata and is
therefore 4.5 bits per value for HIF4.

\paragraph{FlashAttention tiling.}
Standard FlashAttention maintains a row-wise running maximum $m_i$ and
normalizer $\ell_i$ updated per tile.
For each tile $(i,j)$: update the running max
$m_i^{\mathrm{new}}=\max(m_i,\mathrm{rowmax}(S_{ij}))$;
rescale history
$O_i\leftarrow O_i\exp(m_i-m_i^{\mathrm{new}})$,
$\ell_i\leftarrow\ell_i\exp(m_i-m_i^{\mathrm{new}})$,
$m_i\leftarrow m_i^{\mathrm{new}}$;
then compute $P_{ij}=\exp(S_{ij}-m_i)$,
accumulate $O_i\mathrel{+}=P_{ij}V_j$,
$\ell_i\mathrel{+}=\mathrm{rowsum}(P_{ij})$,
and finalize $O_i\leftarrow O_i/\ell_i$.
The vector path handles the softmax, rescaling, and normalizer;
the Cube path handles the two GEMMs ($QK^T$ and $PV$).
In method descriptions we write ``$P_{ij}=\exp(S_{ij}-m_i)$'' with $m_i$
denoting the post-rescaling running maximum for the current tile.

\paragraph{C4V16-Aux.}
Our scheme (C4V16-Aux) quantizes both $QK^T$ and $PV$ to HIF4.
The online softmax state is FP16.
An auxiliary Q-Mean GEMV~\citep{zhang2025sage2}, acknowledged as
adopted from prior work, operates in FP16 at $<0.75\%$ of
$QK^T$ Cube cost.

\section{Method}
\label{sec:method}

\subsection{K-Activation Outliers: Origin and Structural Properties}
\label{sec:topology}

A necessary condition for accurate 4-bit GEMM computation is that the
input activations lie within a range that the quantizer's group scale
can represent with sufficient resolution.
HIF4 assigns a level-1 base scale to every contiguous block of 64
elements (refined by level-2/level-3 micro-exponents): if one element in
a group has a magnitude substantially larger than the rest, the base
scale is dominated by that element and the remaining elements are
quantized with effectively fewer useful bits, beyond what the per-element
micro-exponents recover.
We begin by characterizing the structure of the K-activation outliers
that violate this condition in practice.
In practice, a small number of channels in K can carry magnitudes an
order of magnitude larger than the bulk, forcing the group base scale to
accommodate the extreme value at the cost of coarser quantization for
the other channels in the group.
Although such sparse outliers leave aggregate accuracy largely intact,
they enlarge the per-tile attention error and thereby amplify
\emph{decision-level} drift---the prediction flips and regressions that
our metrics (Section~\ref{sec:exp}) are designed to expose---which is the
effect Smooth-QK targets.
Understanding \emph{where} this outlier originates is essential for
designing a correction that is both effective and deployment-friendly.

Figure~\ref{fig:causal} traces the computation path leading to the
outlier at Layer~0, channel~51 of Qwen3-8B.
The output of the K-projection layer is well-behaved, with magnitude
approximately~1.08.
However, QK-RMSNorm then divides by a small per-token RMS value and
multiplies by the per-channel affine weight $\gamma_{K,51}$; with
$\gamma_{K,51}\approx34$, this amplification alone pushes the channel
magnitude to approximately~155, saturating HIF4's group base scale.

\begin{figure}[htbp]
\centering
\includegraphics[width=0.66\linewidth]{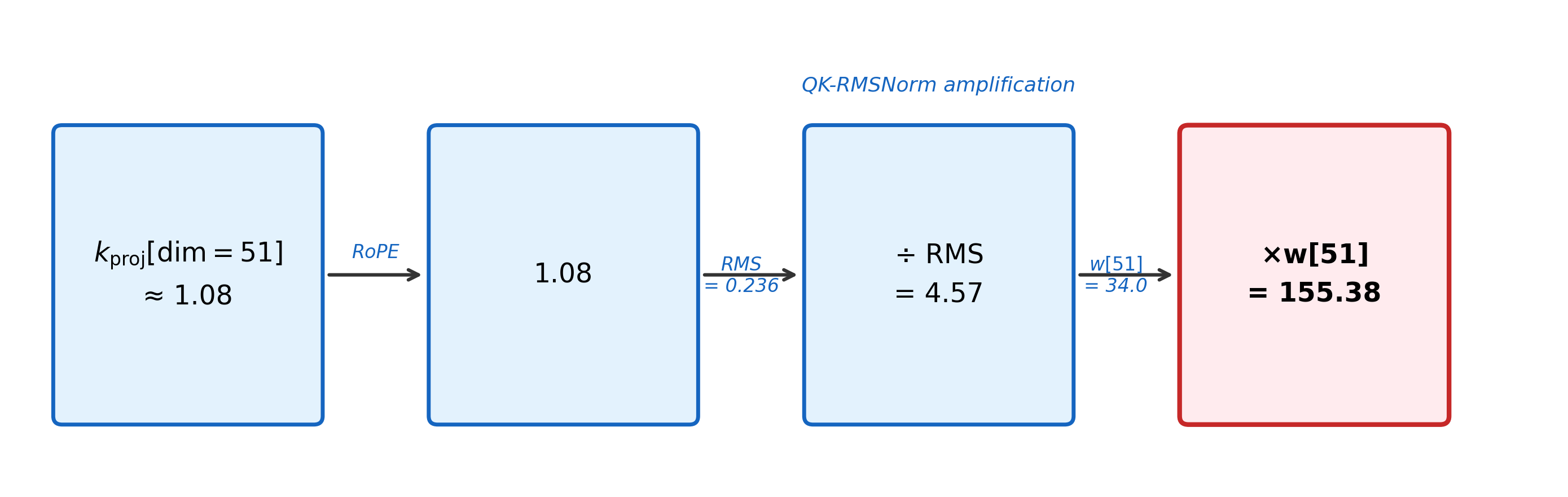}
\caption{Computation chain leading to the K-activation outlier at
Layer~0, channel~51 of Qwen3-8B.
The projection output is normal ($\approx$1.08); the large value arises
entirely within QK-RMSNorm through the combination of a small per-token
RMS normalizer and the large affine weight $\gamma_{K,51}\approx34$.}
\label{fig:causal}
\end{figure}

The key observation is that $\gamma_{K,51}$ is a \emph{fixed model
parameter}, not a property of any particular input token.
Figure~\ref{fig:weight51} shows $\gamma_{K,51}$ across all 36 layers:
it is an isolated outlier at Layer~0, roughly $12\times$ larger than
the layer mean ($=2.87$).
All other layers carry values within normal range, which means the
quantization difficulty is highly concentrated and not a generic property
of the model's arithmetic.

\begin{figure}[htbp]
\centering
\begin{minipage}[t]{0.4\linewidth}
\centering
\includegraphics[width=\linewidth]{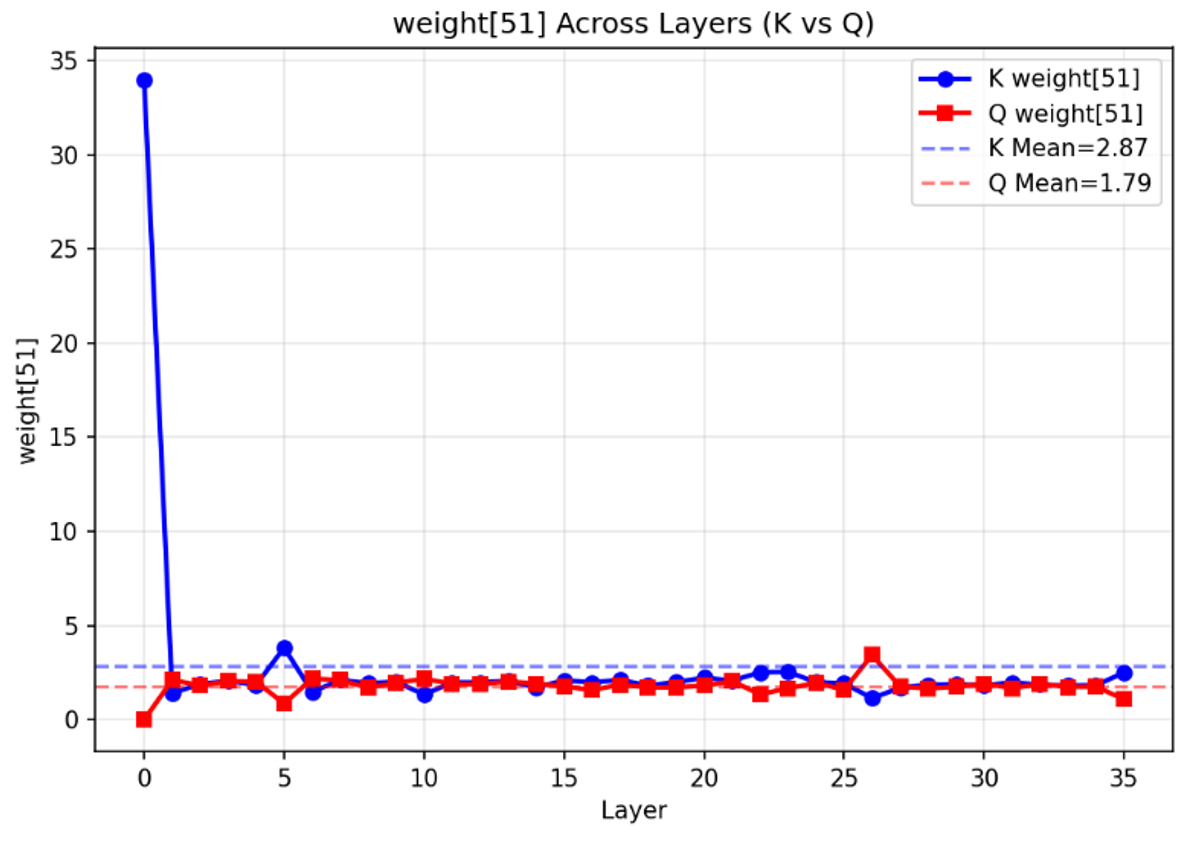}
\caption{QK-RMSNorm affine weight at channel~51 for K (blue) and Q
(red) across 36 layers of Qwen3-8B.
Layer~0 K weight ($\approx34$) is far outside the range of all other
values (K layer mean $=2.87$, Q layer mean $=1.79$).}
\label{fig:weight51}
\end{minipage}
\hfill
\begin{minipage}[t]{0.53\linewidth}
\centering
\includegraphics[width=\linewidth]{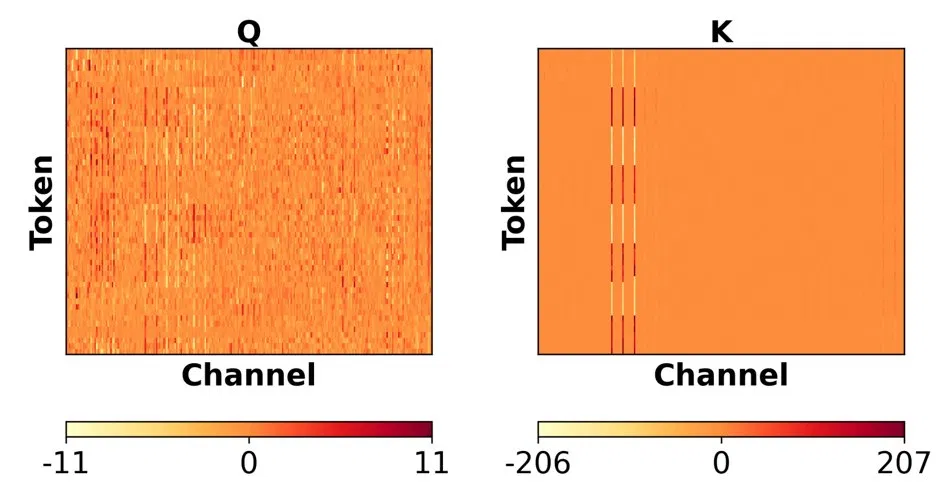}
\caption{Q (left) and K (right) head-by-channel maximum-magnitude
heatmap at Layer~0.
K shows a uniform vertical stripe at channel~51 across all heads,
consistent with a parameter-level broadcast cause.
Q shows no comparable pattern.}
\label{fig:heatmap}
\end{minipage}
\end{figure}

Because $\gamma_{K,51}$ is a fixed parameter, the resulting outlier
is broadcast identically to every token and---after GQA expansion---to
every attention head.
Figure~\ref{fig:heatmap} visualizes this at the activation level: the
K heatmap shows a \emph{head-invariant vertical stripe} at channel~51,
uniform across all 32 heads, while Q exhibits a sparse and
head-varying pattern with no comparable concentration.
This spatial signature is diagnostic of a parameter-level broadcast
cause: if the outlier were input-dependent, its magnitude would vary
across heads and tokens.
The head-invariance observed here means that a single per-channel
scalar $\hat{s}_{51}$ is sufficient to suppress the outlier for the
entire layer, without needing per-head or per-token corrections.

\begin{table}[htbp]
\centering\small
\caption{Outlier diagnostic statistics across five models.
$\rho_K=\max|K|/P99(|K|)$: intra-K outlier severity.
$T_{K/Q}=\max|K|/\max|Q|$: K-to-Q dynamic-range asymmetry.
Both metrics are max-layer values.
Jaccard: outlier channel stability across inputs.
The gate applies Smooth-QK only when K-outliers are both concentrated
($\rho_K\geq6$) and dominant over Q ($T_{K/Q}\geq2$); within that region
$T_{K/Q}>3$ selects Strong and $2\leq T_{K/Q}\leq3$ selects Mild.
All other models are bypassed.
Thresholds are calibrated on these five models (Section~\ref{sec:threshold})
and validated against the per-model ablations.}
\label{tab:outlier}
\setlength{\tabcolsep}{3.5pt}
\begin{tabular}{lrrrll}
\toprule
Model & $\rho_K$ & $T_{K/Q}$ & Jaccard & Gate condition & Action\\
\midrule
Qwen3-8B    & 14.15 & 14.47 & 0.707 & $\rho_K\geq6$ \& $T_{K/Q}>3$ & Strong ($\alpha=0.5$)\\
Gemma2-9B   &  8.11 &  2.17 & 0.676 & $\rho_K\geq6$ \& $2\leq T_{K/Q}\leq3$ & Mild ($\alpha=0.25$)\\
LLaMA3.1-8B &  4.78 &  2.70 & 0.679 & $\rho_K<6$                  & Bypass ($s=1$)\\
Mistral-7B  &  3.56 &  1.63 & 0.887 & $\rho_K<6$ \& $T_{K/Q}<2$   & Bypass ($s=1$)\\
Phi-4B      &  3.69 &  1.50 & 0.786 & $\rho_K<6$ \& $T_{K/Q}<2$   & Bypass ($s=1$)\\
\bottomrule
\end{tabular}
\end{table}

Table~\ref{tab:outlier} quantifies the outlier across five models.
Figure~\ref{fig:kstats} shows the layer-wise K profile for three
representative models (Qwen3-8B, Gemma2-9B, LLaMA3.1-8B); the two bypass
models Mistral-7B and Phi-4B have flat profiles with no dominant layer
and are omitted.

\begin{figure}[htbp]
\centering
\begin{subfigure}[b]{0.49\linewidth}
\includegraphics[width=\linewidth]{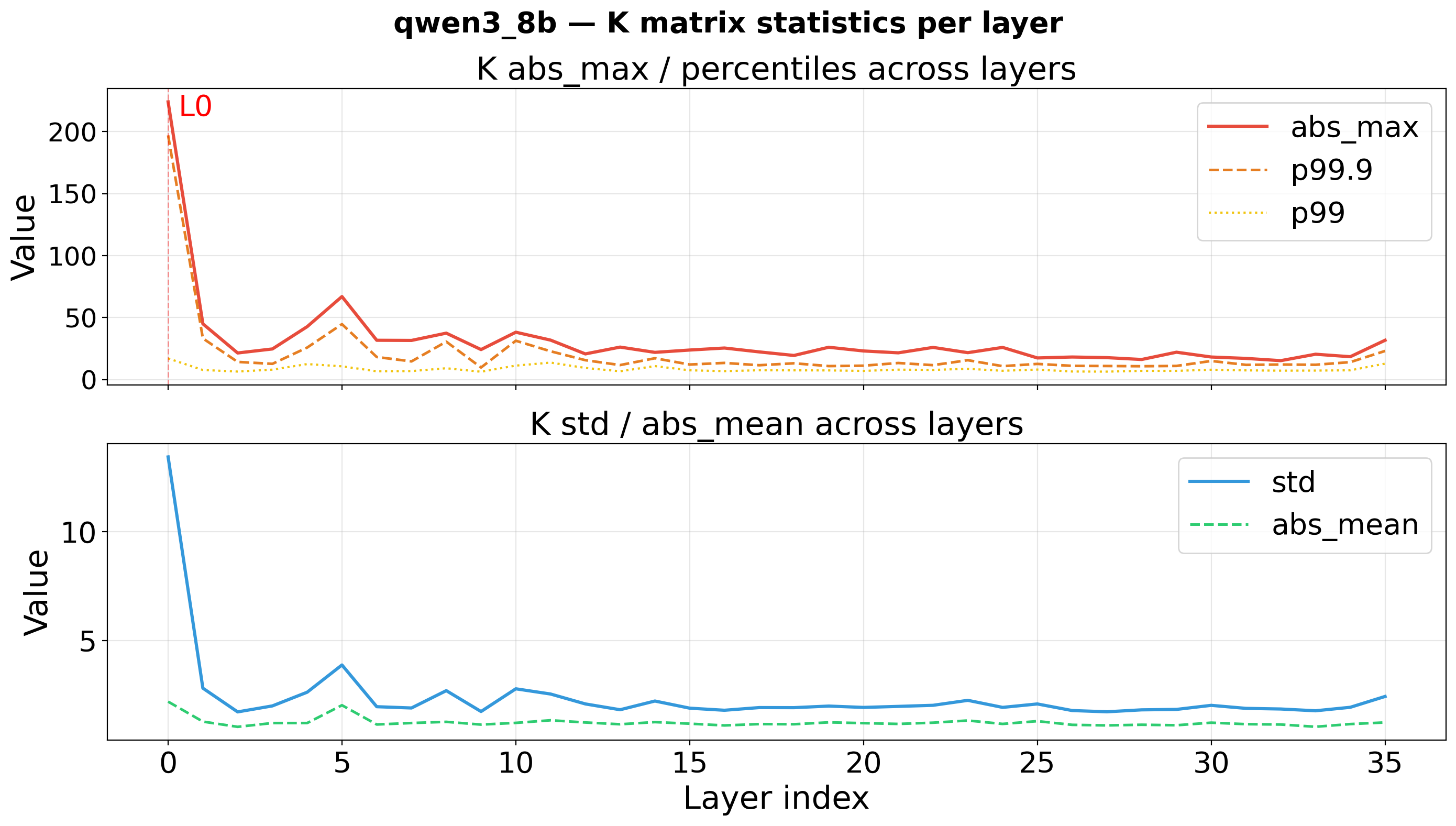}
\caption{Qwen3-8B}
\end{subfigure}
\begin{subfigure}[b]{0.49\linewidth}
\includegraphics[width=\linewidth]{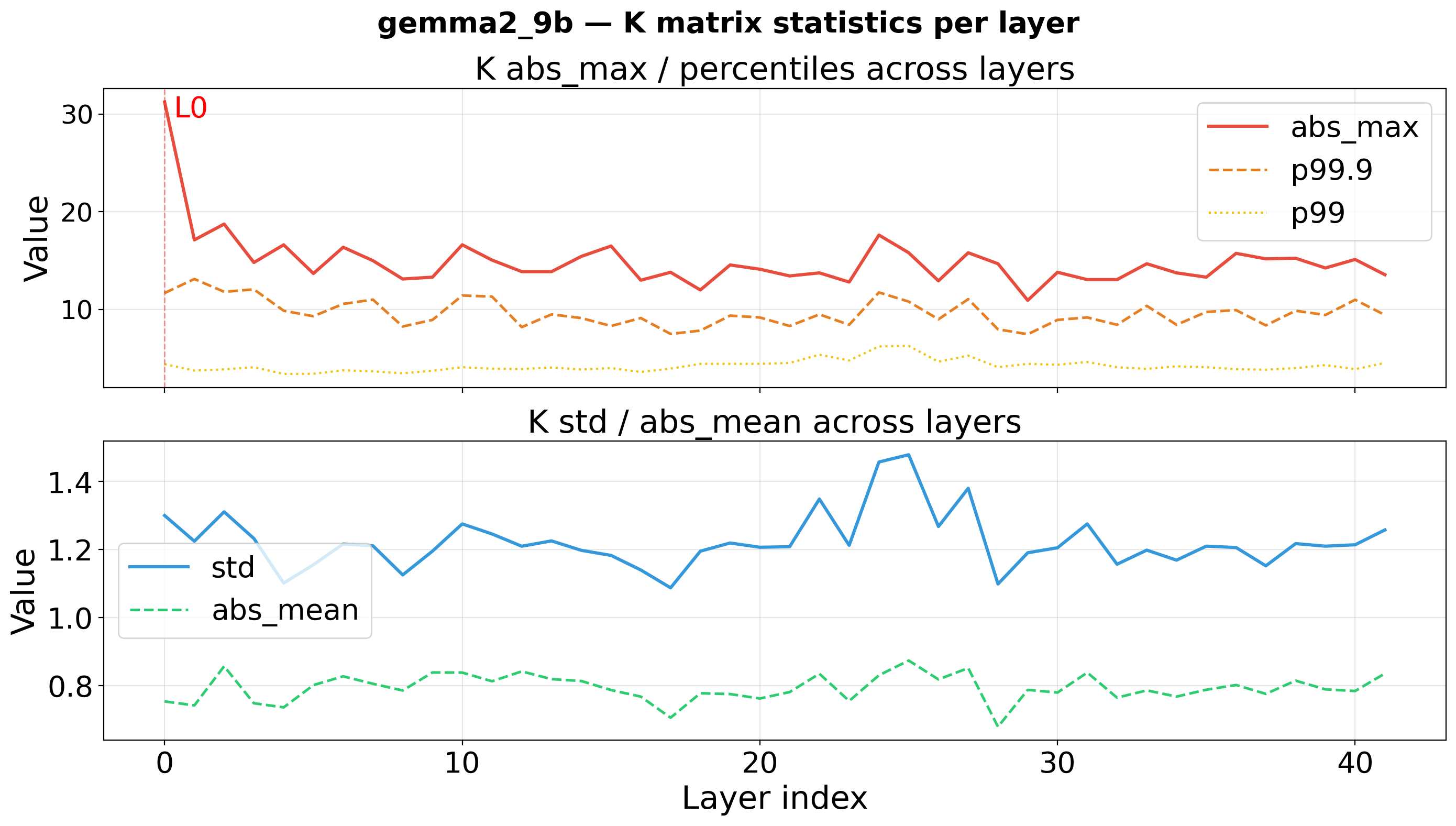}
\caption{Gemma2-9B}
\end{subfigure}
\begin{subfigure}[b]{0.48\linewidth}
\includegraphics[width=\linewidth]{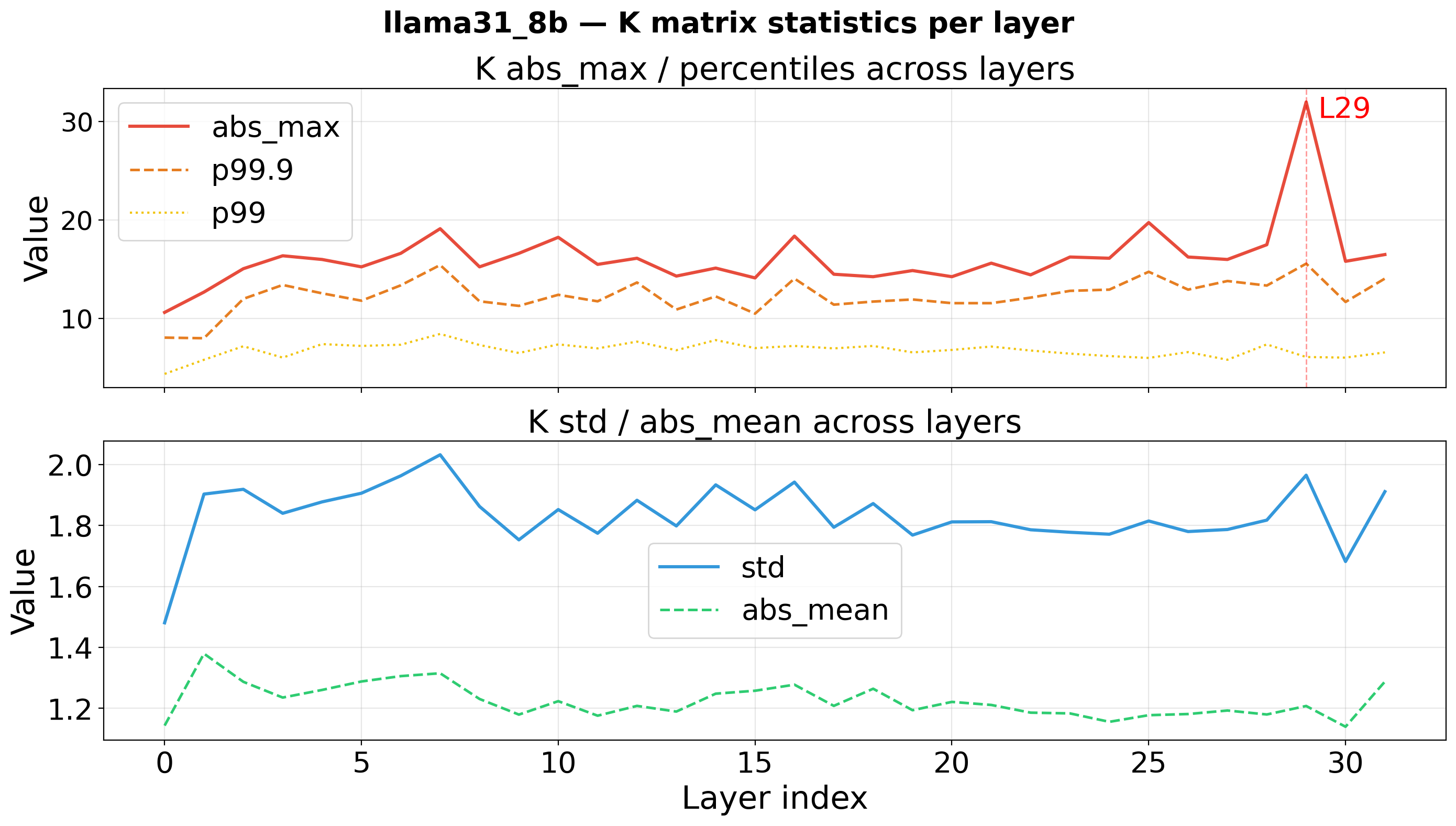}
\caption{LLaMA3.1-8B}
\end{subfigure}
\caption{Layer-wise K statistics (max, p99.9, p99, std) for three models.
Qwen3-8B has an isolated dominant Layer-0 spike; Gemma2-9B has a moderate
Layer-0 peak; LLaMA3.1-8B has no concentrated outlier layer.}
\label{fig:kstats}
\end{figure}

\subsection{Smooth-QK}
\label{sec:sqk}

Applied post-RoPE:
\begin{equation}
\tilde{q}_c = q_c \times \hat{s}_c,\quad
\tilde{k}_c = k_c \div \hat{s}_c,\quad
\hat{s}_c = \frac{|K_{\max,c}|^\alpha}{|Q_{\max,c}|^{1-\alpha}},
\quad \alpha \in \{0.5, 0.25\}.
\label{eq:sqk}
\end{equation}
In matrix form: $\tilde{Q}=Q\,\mathrm{diag}(\hat{s})$,
$\tilde{K}=K\,\mathrm{diag}(\hat{s})^{-1}$, preserving
$\tilde{Q}\tilde{K}^T=QK^T$.
The exponent $\alpha$ controls how much difficulty is transferred from K
to Q and is set by the gate (Section~\ref{sec:threshold}): the Strong
branch uses $\alpha=0.5$ and the Mild branch $\alpha=0.25$.
With $\alpha=0.5$, $\hat{s}_c=\sqrt{|K_{\max,c}|/|Q_{\max,c}|}$, which
equalizes the transformed magnitudes $|K_{\max,c}|/\hat{s}_c$ and
$|Q_{\max,c}|\,\hat{s}_c$ at their common value
$\sqrt{|K_{\max,c}|\,|Q_{\max,c}|}$ (the geometric mean of the two
magnitudes), the choice that minimizes
$\max(|\tilde{k}_c|,|\tilde{q}_c|)$.
A smaller $\alpha$ (Mild) transfers proportionally less difficulty to Q,
which is preferable when $T_{K/Q}$ is moderate.

Equation~\eqref{eq:sqk} is defined \emph{per channel} on the calibration
statistics $|K_{\max,c}|,|Q_{\max,c}|$ aggregated over the calibration
set (Section~\ref{sec:static}).
The magnitudes quoted in Figure~\ref{fig:smooth_effect} are a direct
visualization of one fixed attention slice (Layer~0, a single batch
element and a single head, first 64 tokens) and are \emph{not} the
inputs to Eq.~\eqref{eq:sqk}: ``before'' is the measured max-abs of that
slice's K, and ``after'' is the measured max-abs of the same slice
\emph{after} Smooth-QK has actually been applied, not a value obtained by
dividing the before-value by a scale factor.
Because the channel-51 outlier magnitude varies across heads, tokens, and
inputs, the value seen in this particular slice (104.5) differs from the
calibration-aggregate maximum reported in Table~\ref{tab:smooth_layers}
(220.8) and from the single-token forward trace in
Figure~\ref{fig:causal} ($\approx$155); these are three different
statistics of the same underlying parameter-driven outlier and are not
expected to coincide.
Appendix~\ref{app:smooth_dist} states the exact statistic behind each
reported number.

\begin{figure}[htbp]
\centering
\begin{subfigure}[b]{0.49\linewidth}
\includegraphics[width=\linewidth]{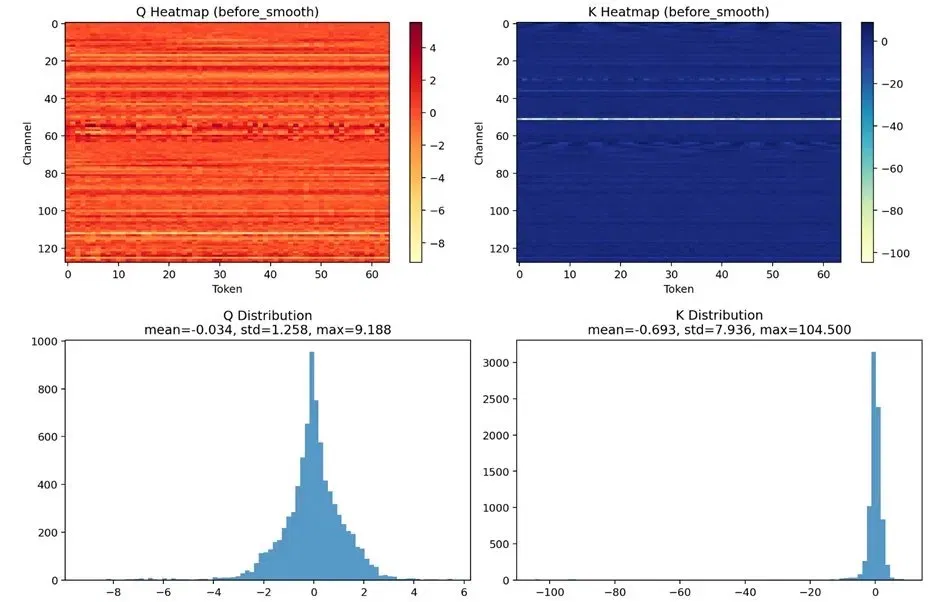}
\caption{Before: K max-abs $=104.5$.}
\end{subfigure}
\begin{subfigure}[b]{0.49\linewidth}
\includegraphics[width=\linewidth]{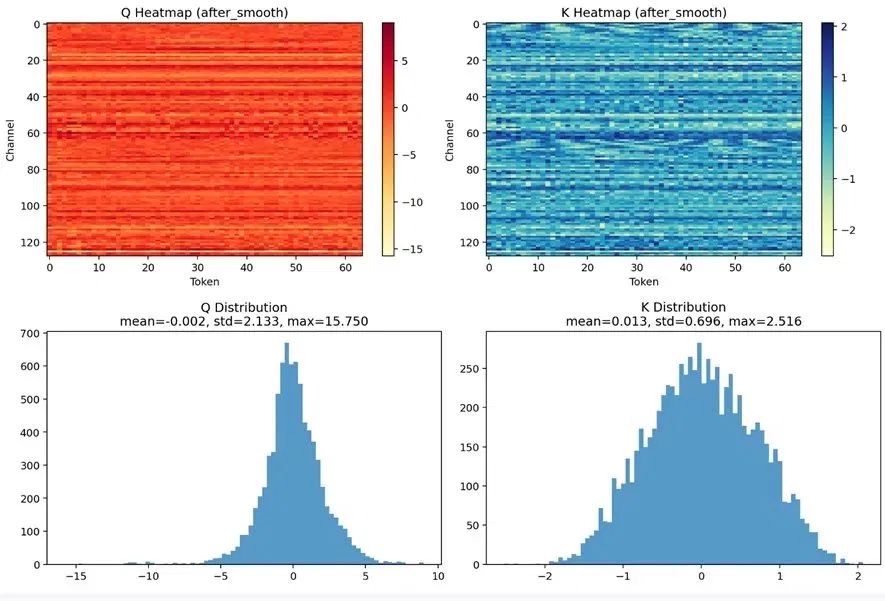}
\caption{After: K max-abs $=2.516$.}
\end{subfigure}
\caption{K activations before and after Smooth-QK for one fixed attention
slice (Qwen3-8B, Layer~0, a single batch element, head~0, first 64
tokens).
``After'' is the measured max-abs of the slice once Smooth-QK has been
applied, not a value computed from the scale factor.
The channel-51 stripe is suppressed; K becomes amenable to 4-bit
quantization.
This per-slice max-abs is one of three distinct statistics of the same
outlier used in the paper (per-slice here, calibration-aggregate in
Table~\ref{tab:smooth_layers}, single-token trace in
Figure~\ref{fig:causal}); see Appendix~\ref{app:smooth_dist}.}
\label{fig:smooth_effect}
\end{figure}

Smooth-QK is layer- and channel-agnostic: the calibration procedure
evaluates every layer and channel.
The Qwen3-8B Layer-0 case is the clearest motivating example but not
a structural assumption.

\begin{figure}[t]
\centering
\includegraphics[width=0.53\linewidth]{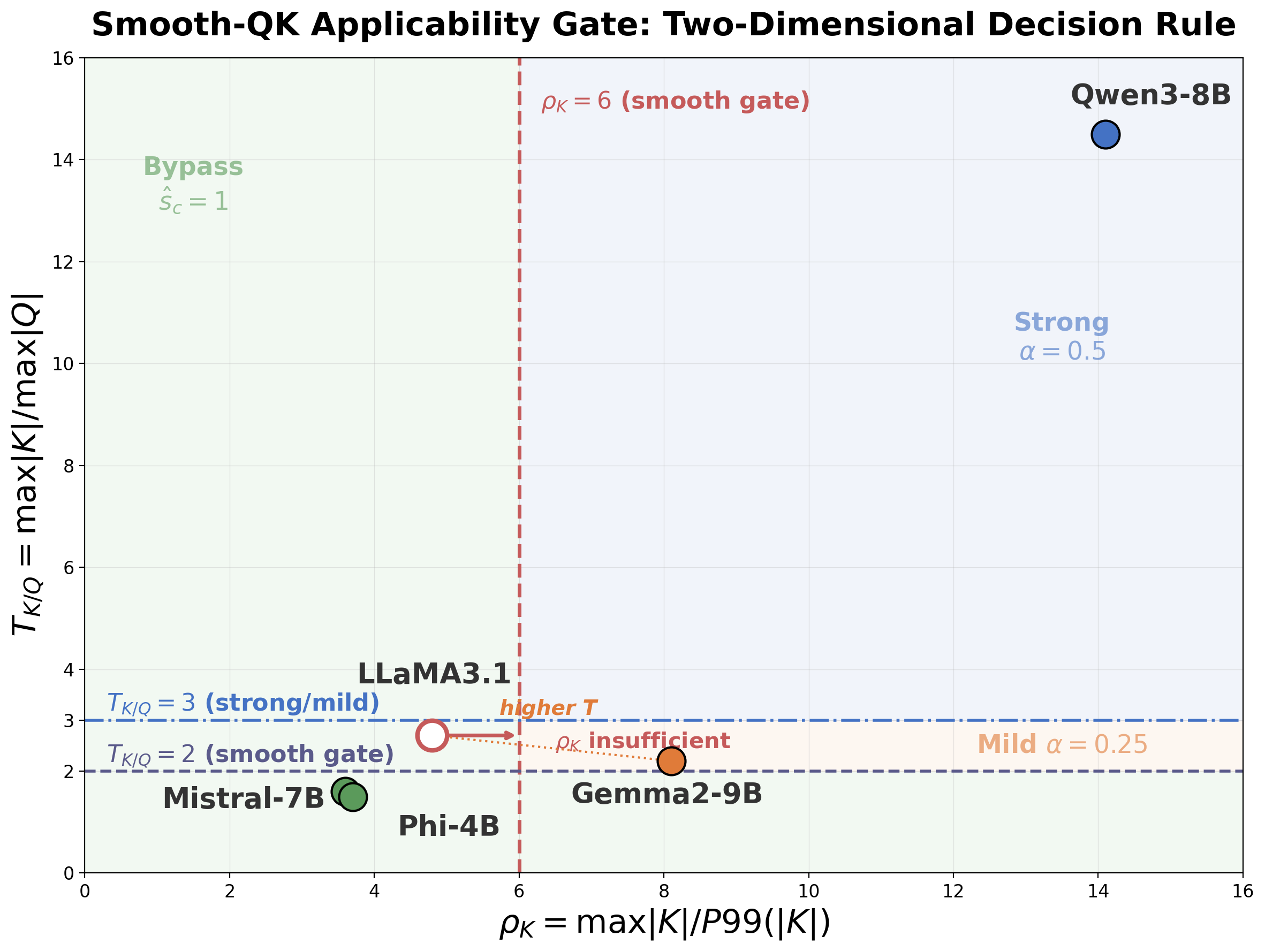}
\caption{The Smooth-QK applicability gate as a two-dimensional decision
plane over $\rho_K$ (intra-K outlier concentration) and $T_{K/Q}$
(K-to-Q dominance), with all five models plotted at their calibration
values.
Smoothing is enabled only in the upper-right region where \emph{both}
$\rho_K\geq6$ and $T_{K/Q}\geq2$ hold; within it, $T_{K/Q}>3$ selects
Strong ($\alpha=0.5$, Qwen3-8B) and $2\leq T_{K/Q}\leq3$ selects Mild
($\alpha=0.25$, Gemma2-9B). All other models are bypassed
($\hat{s}=1$).
LLaMA3.1-8B (open marker) is the informative boundary case: its
$T_{K/Q}=2.70$ exceeds Gemma2-9B's $2.17$, yet its $\rho_K=4.78$ falls
short of the $\rho_K\geq6$ threshold, so the gate correctly bypasses
it---illustrating that $T_{K/Q}$ alone is insufficient and both
coordinates are required.
Jaccard stability (Appendix~\ref{app:jaccard}) validates static
calibration but is not a gate trigger: Mistral-7B has Jaccard\,$=0.887$
yet is bypassed because $\rho_K=3.56<6$ and $T_{K/Q}=1.63<2$.}
\label{fig:adaptive}
\end{figure}

\subsection{Calibration-Static Deployment}
\label{sec:static}

Since the outlier originates from fixed model parameters, $\hat{s}_c$
is stable across inputs (inter-sample CV $<8\%$ for high-outlier channels).
Calibration uses 1024 samples from training/auxiliary splits only.
The per-channel statistics $|K_{\max,c}|$ and $|Q_{\max,c}|$ in
Eq.~\eqref{eq:sqk} are RMS-aggregated over the calibration set
(Section~\ref{sec:calib}); this aggregate is the statistic used to
compute every $\hat{s}_c$, and it is distinct from the per-slice max-abs
values used purely for visualization in Figure~\ref{fig:smooth_effect}.
At inference, applying Smooth-QK requires one element-wise multiply and
divide per channel per token per layer---no per-tile online reduction.

\subsection{Adaptive Applicability Gate}
\label{sec:threshold}

Smooth-QK is beneficial only when two conditions hold simultaneously.
First, K must contain sufficiently \emph{concentrated} internal outliers
(high $\rho_K$): if K is internally uniform, its values are already well
represented by the HIF4 group scale and there is nothing to suppress.
Second, K must \emph{dominate} Q (high $T_{K/Q}$): Smooth-QK transfers
quantization difficulty from K to Q, so when $T_{K/Q}$ is low (K and Q
have similar magnitudes), the transfer merely expands Q toward HIF4
saturation without a proportionate benefit on K.
The gate (Figure~\ref{fig:adaptive}) enforces both conditions.

We treat the gate as an \emph{offline diagnostic}: the per-channel
statistics $\rho_K$ and $T_{K/Q}$ are computed once from the calibration
set, and the resulting branch (Strong / Mild / Bypass) is fixed for
deployment. We do not perform any runtime model selection.
Concretely, smoothing is enabled only when \emph{both} conditions hold,
$\rho_K\geq6$ and $T_{K/Q}\geq2$; within that region $T_{K/Q}>3$ selects
strong smoothing ($\alpha=0.5$) and $2\leq T_{K/Q}\leq3$ selects mild
smoothing ($\alpha=0.25$, which transfers less difficulty to Q). Any
model failing either condition is bypassed ($\hat{s}=1$).
Both conditions are necessary and neither alone suffices: $T_{K/Q}$
captures whether transferring difficulty from K to Q has a worthwhile
target, while $\rho_K$ captures whether K's outliers are concentrated
enough to be worth transferring. LLaMA3.1-8B illustrates why $T_{K/Q}$
alone is insufficient---its $T_{K/Q}=2.70$ actually exceeds Gemma2-9B's
$2.17$, yet its low $\rho_K=4.78$ (K-outliers too diffuse) correctly
routes it to bypass.
On our five models this routes Qwen3-8B to Strong, Gemma2-9B to Mild,
and LLaMA3.1-8B, Mistral-7B, and Phi-4B to Bypass. As we show in
Section~\ref{sec:crossmodel} and Appendix~\ref{app:smooth_dist}, each of
these routing decisions coincides with the empirically best configuration
for that model: smoothing helps Qwen3-8B and Gemma2-9B, and is neutral or
harmful on the three bypass models. The thresholds are empirical
boundaries fitted to these five models; we discuss their limited
statistical support in Section~\ref{sec:limitations}.

\subsection{P-Reordering}
\label{sec:preorder}

In our Direct HIF4 baseline, following the common emulation path used by
several FP4 attention implementations, the softmax normalizer is
accumulated from a higher-precision reconstruction $\tilde{P}$ of the
quantized weights, while the $PV$ GEMM consumes the quantized
$\hat{P}$ directly.
This inconsistency introduces a systematic error that Theorem~\ref{thm:main}
characterizes exactly.
P-Reordering replaces $\tilde{P}$ in the normalizer with $\hat{P}$,
making both paths consistent.

\begin{theorem}
\label{thm:main}
Let $a_j=\exp(S_j-m)$ be the unnormalized softmax weight at position
$j$, $Z=\sum_k a_k$ the partition sum, and $w_j=a_j/Z$ the normalized
weight.
Let $\hat{a}_j=a_j(1+\varepsilon_j)$ be the HIF4-quantized value.
Define $\bar\varepsilon=\sum_j w_j\varepsilon_j$,
$\delta_j=\varepsilon_j-\bar\varepsilon$ ($\sum_j w_j\delta_j=0$),
$O^*=\sum_j w_jv_j$ for fixed $v_j$, and
$e_t=\sum_j w_j\delta_j(v_j-O^*)$.  Then:
\begin{align}
O_{\mathrm{con}}-O^* &= \frac{e_t}{1+\bar\varepsilon}
  &&\!\text{(P-Reordering: denominator\,$=\textstyle\sum_k\hat{a}_k$)}\\
O_{\mathrm{mis}}-O^* &= e_t + \bar\varepsilon\,O^*
  &&\!\text{(mismatched: denominator\,$=Z$)}
\end{align}
\emph{Scope}: $V$ is held fixed; Q/K/V quantization errors are
independent and not modelled here.
\end{theorem}
\begin{proof}
\textit{Mismatched:}
$O_{\mathrm{mis}}=\frac{\sum_j\hat{a}_jv_j}{Z}
=O^*+\sum_j w_j\varepsilon_jv_j$.
For the last sum, write $\varepsilon_j=\bar\varepsilon+\delta_j$:
$\sum_j w_j\varepsilon_jv_j
=\bar\varepsilon\sum_j w_jv_j + \sum_j w_j\delta_j v_j
=\bar\varepsilon O^* + \sum_j w_j\delta_j(v_j-O^*) + O^*\!\underbrace{\sum_j w_j\delta_j}_{=0}
=\bar\varepsilon O^* + e_t$.
Hence $O_{\mathrm{mis}}-O^*=e_t+\bar\varepsilon O^*$.
\textit{Consistent:}
$O_{\mathrm{con}}=\frac{\sum_j\hat{a}_jv_j}{\sum_k\hat{a}_k}
=\frac{Z(1+\bar\varepsilon)O^*+Ze_t}{Z(1+\bar\varepsilon)}
=O^*+\frac{e_t}{1+\bar\varepsilon}$.
\end{proof}

The term $\bar\varepsilon O^*$ scales the attention output coherently
along the $O^*$ direction---a radial error that standard metrics such
as relative $L_2$ distance are insensitive to because it is small
compared to the Q/K/V quantization noise floor.
Theorem~\ref{thm:main} isolates the $P$-quantization contribution under
fixed $V$; it is an exact local statement about the normalizer path, not
a claim that this term dominates the full end-to-end error budget, which
also contains independent Q/K/V noise.
The end-to-end ablation in Section~\ref{sec:ablation} (rows B vs.\ D and
C vs.\ E) measures the realized benefit when all of Q, K, V, and $P$ are
quantized; a controlled $P$-only local-operator study that holds Q, K, V
at BF16 is left to a future version.

\begin{figure}[t]
\centering
\begin{minipage}[t]{0.48\linewidth}
\centering
\includegraphics[width=0.6\linewidth]{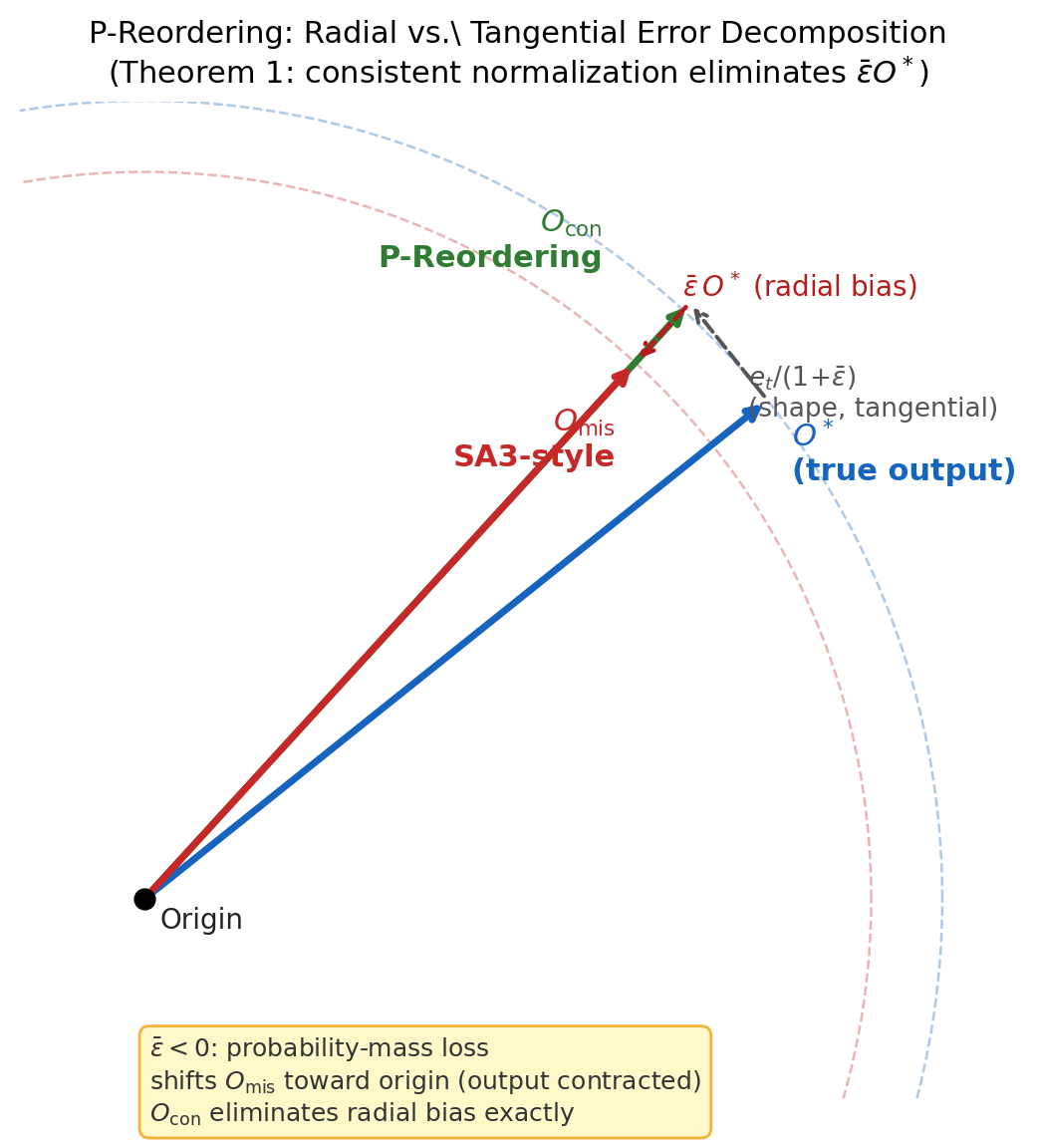}
\caption{Theorem~\ref{thm:main} illustrated.
$O_{\mathrm{mis}}$ (red) is contracted toward the origin by the
$P$-normalizer-induced term $\bar\varepsilon O^*$ ($\bar\varepsilon<0$
under 4-bit).
$O_{\mathrm{con}}$ (P-Reordering, green) does not exhibit this
contraction under fixed $V$.}
\label{fig:preorder_schematic}
\end{minipage}
\hfill
\begin{minipage}[t]{0.48\linewidth}
\centering
\includegraphics[width=0.95\linewidth]{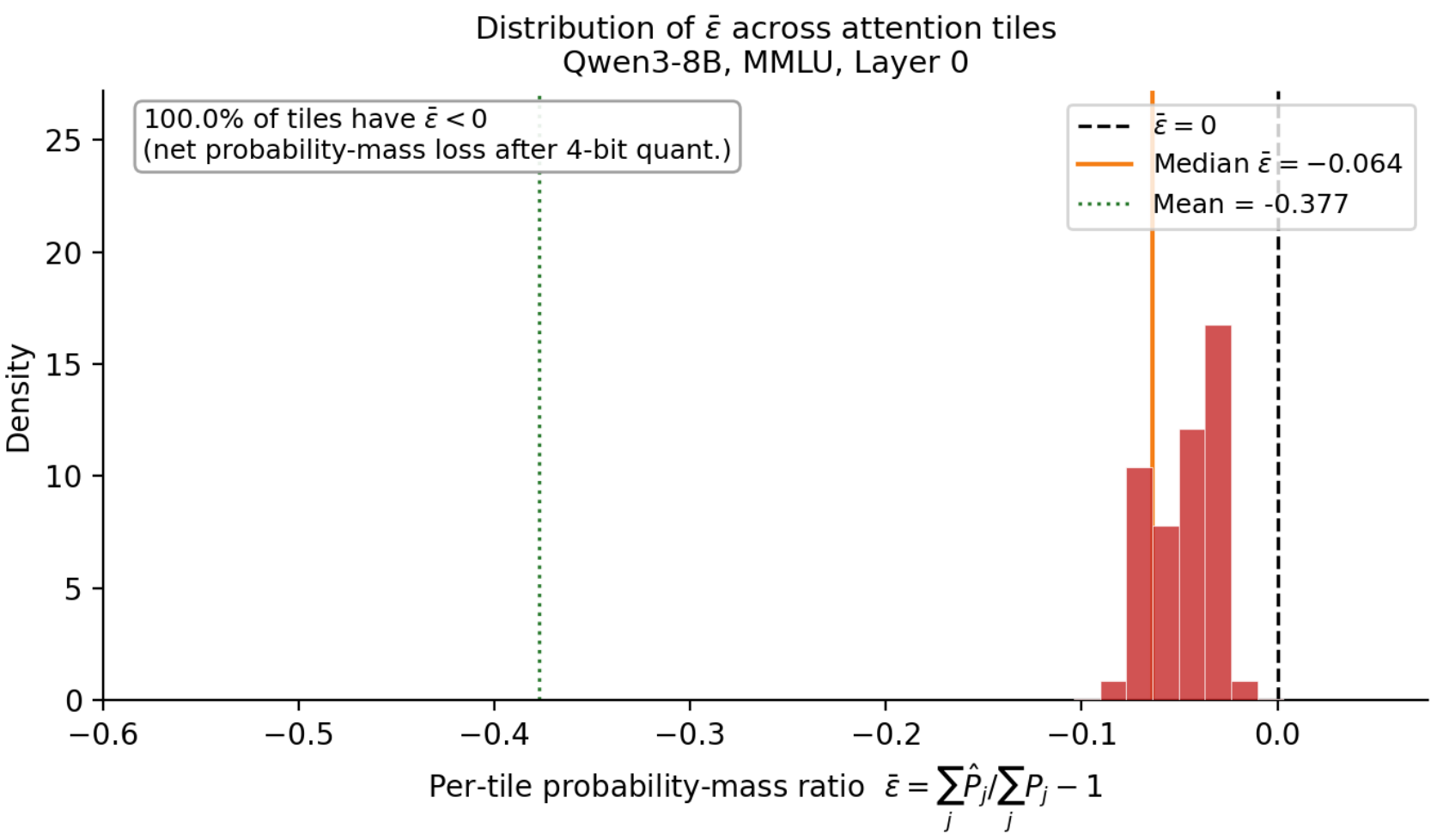}
\caption{Distribution of per-tile $\bar\varepsilon=\sum_j\hat{a}_j/Z-1$
across 3.6~million attention tiles from Qwen3-8B running MMLU at Layer~0
(direct HIF4 quantization, no Smooth-QK).
100\% of tiles have $\bar\varepsilon<0$ (median $=-0.064$),
confirming the consistent negative sign of the bias identified in
Theorem~\ref{thm:main} in this measured Layer-0 Qwen3-8B trace.
The \emph{end-to-end} magnitude contribution across all layers is a
separate question, addressed by the D-vs-B and E-vs-C ablations
(Section~\ref{sec:ablation}) rather than by this per-tile distribution.}
\label{fig:masshist}
\end{minipage}
\end{figure}

\paragraph{Empirical validation.}
Figure~\ref{fig:masshist} shows the distribution of $\bar\varepsilon$
across 3.6~million tiles from Qwen3-8B Layer~0.
Every tile exhibits net probability-mass loss ($\bar\varepsilon<0$).
The radial error term $\bar\varepsilon O^*$ has a consistent
\emph{sign} in this trace; because the direction of $O^*$ varies across
layers, heads, and tokens, the global cross-layer effect is assessed
empirically via the D-vs-B and E-vs-C ablations
(Section~\ref{sec:ablation}) rather than extrapolated from the theorem.

\paragraph{Secondary benefit.}
Accumulating $\sum_j\hat{a}_j$ alongside $\sum_j\hat{a}_jv_j$ is
achievable within the same Cube GEMM by appending a column of ones
to $V$: $\hat{P}[\hat{V}\mid\mathbf{1}]=[\hat{P}\hat{V}\mid\hat{P}\mathbf{1}]$.
This fuses the normalizer into the $PV$ Cube GEMM, removing a separate
vector-path reduction.

\subsection{Auxiliary Q-Mean Compensation}
\label{sec:qmean}
We include the Q-Mean compensation from~\citet{zhang2025sage2} as an
acknowledged auxiliary; it addresses Q DC-bias and is not a
contribution of this paper.

\subsection{Complete Algorithm}
\label{sec:algorithm}
\begin{algorithm}[H]
\small
\caption{C4V16-Aux (HiFA4)}\label{alg:main}
\begin{algorithmic}[1]
\Require Q,K,V (FP16, post-RoPE); static $\hat{s}\in\mathbb{R}^{d_k}$
\State $\tilde{Q}\leftarrow Q\,\mathrm{diag}(\hat{s})$;\;
       $\tilde{K}\leftarrow K\,\mathrm{diag}(\hat{s})^{-1}$
       \Comment{Smooth-QK}
\State $\mu_i\leftarrow\frac{1}{d}\sum_{c=1}^{d}\tilde{Q}_{i,c}$;\;
       $Q_{z,i,c}\leftarrow\tilde{Q}_{i,c}-\mu_i$
       \Comment{Q-Mean: per-token scalar mean over head-feature dim $d$}
\State $\hat{Q},\hat{K},\hat{V}\leftarrow\mathrm{HIF4}(Q_z,\tilde{K},V)$
\For{each tile $(i,j)$}
  \State $S_{ij}\leftarrow\hat{Q}_i\hat{K}_j^T/\!\sqrt{d}$
    \Comment{HIF4 Cube}
  \State $S_{ij}\mathrel{+}=\mu_i\sum_{c}\tilde{K}_{j,c}^T$
    \Comment{FP16 GEMV (Q-Mean correction)}
  \State $m_i^{\mathrm{new}}\leftarrow\max(m_i,\mathrm{rowmax}(S_{ij}))$
    \Comment{online softmax: update running max}
  \State $r\leftarrow\exp(m_i-m_i^{\mathrm{new}})$;\;
         $O_i\leftarrow r\cdot O_i$;\;
         $\ell_i\leftarrow r\cdot\ell_i$;\;
         $m_i\leftarrow m_i^{\mathrm{new}}$
    \Comment{rescale history}
  \State $P_{ij}\leftarrow\exp(S_{ij}-m_i)$
  \State $\hat{P}_{ij}\leftarrow\mathrm{HIF4}(P_{ij})$
    \Comment{P-Reordering: quantize \emph{before} normalizer accumulation}
  \State $[O_i\mid\ell_i]\mathrel{+}=\hat{P}_{ij}[\hat{V}_j\mid\mathbf{1}]$
    \Comment{HIF4 Cube, fused normalizer}
\EndFor
\State $O_i\leftarrow O_i/\ell_i$
\end{algorithmic}
\end{algorithm}

\noindent\textbf{P-Reordering under online softmax.}
Theorem~\ref{thm:main} holds for any fixed softmax offset $m$; in
Algorithm~\ref{alg:main}, $m_i$ is the running maximum that has already
been updated (line~7) before $\hat{P}_{ij}$ is computed (line~9).
The rescaling of $\ell_i$ in line~8 applies the same factor to the
accumulated normalizer as to the accumulated output $O_i$, so the
P-Reordering consistency property---that numerator and denominator both
use the same quantized $\hat{P}$---is preserved across all tiles
independently of the per-tile rescaling.

\section{Experiments}
\label{sec:exp}

Throughout all experiments, the \textbf{BF16 baseline} is the
unmodified FlashAttention execution with BF16 $QK^T$ and $PV$ GEMMs
and FP16 online softmax state (labeled C16V16 in
Section~\ref{sec:latency}).
In all quantized configurations, the surrounding model weights,
prompts, scoring protocol, and non-attention computation remain
unchanged; only the attention operator is replaced by the corresponding
Direct HIF4 or HiFA4 variant defined in Section~\ref{sec:method}.

We evaluate on Qwen3-8B~\citep{qwen3_2025} (primary),
Gemma2-9B~\citep{gemma2_2024},
LLaMA3.1-8B~\citep{llama3_2024},
Mistral-7B~\citep{jiang2023mistral},
and Phi-4B~\citep{abdin2025phi4mini}.
All results use the official HIF4 QDQ implementation.
Calibration uses only training or auxiliary splits.
Full experimental settings (shot count, prompt templates, scoring
protocol, software versions, calibration provenance, and hook coverage)
are documented in Appendix~\ref{app:setup}.

We report \emph{prediction flip rate} (PFR: fraction of samples where
quantized top-1 differs from BF16) and \emph{Quant Hurts} (BF16-correct
samples that the quantized model answers incorrectly; equivalently, the
per-sample accuracy regressions referred to in the abstract), alongside
accuracy.
These decision-level metrics expose degradation invisible in aggregate accuracy.

\subsection{Calibration Aggregation (Qwen3-8B)}
\label{sec:calib}

\begin{table}[htbp]
\centering\small
\caption{Calibration aggregation strategies on Qwen3-8B.
$\delta_w$: sample-weighted accuracy loss relative to BF16 over the four
decision benchmarks (ARC, TQA, MMLU, HellaSwag), lower is better.
The $\delta_w$ values for BF16 and RMS are computed from the per-column
figures in this table.
For the Direct HIF4 row, $\delta_w$ is marked $(\dagger)$ to indicate
it is reported from the evaluation configuration defined in
Section~\ref{sec:crossmodel} (MMLU $14{,}042$, HellaSwag $10{,}042$,
ARC $295$, TQA $817$), which uses the complete evaluation sets rather
than the mixed-task calibration set shown here; the per-column figures
in this row are from the same configuration but include GPQA and BBH
for comparison.
The $\delta_w$ for Mean and 99th percentile are omitted ($-$) because
these strategies are evaluated only for the purposes of selecting the
calibration aggregator and their $\delta_w$ is not used downstream.
Strategies are ranked primarily by PPL and per-benchmark accuracy;
RMS is selected as it gives the best PPL (9.72) and the smallest
$\delta_w$ (0.70\,pp).}
\label{tab:calib}
\setlength{\tabcolsep}{3pt}
\begin{tabular}{lrrrrrrrr}
\toprule
Strategy & ARC & TQA & GPQA & BBH & Hella & MMLU & $\delta_w$ & PPL\\
\midrule
BF16            &.885&.704&.364&.602&.865&.714& 0   & 9.73\\
Direct HIF4     &.881&.705&.369&.570&.841&.680& $1.12^{\dagger}$&11.35\\
\midrule
Mean            &.881&.716&.364&.596&.849&.699& $-$& 9.83\\
RMS (used)      &.885&.725&.369&.591&.857&.706&\textbf{0.70}&\textbf{9.72}\\
99th percentile &.888&.710&.343&.598&.859&.711& $-$& 9.78\\
\bottomrule
\end{tabular}
\end{table}

RMS aggregation achieves the best $\delta_w=0.70$\,pp and is used
in all subsequent experiments.
The same RMS aggregate over the 1024-sample calibration set defines the
per-channel $|K_{\max,c}|,|Q_{\max,c}|$ used in Eq.~\eqref{eq:sqk}.

\subsection{Main Decision-Consistency Results (Qwen3-8B)}

\begin{table}[htbp]
\centering\small
\caption{Prediction flip rate and Quant Hurts (Qwen3-8B,
Direct HIF4 vs.\ HiFA4).}
\label{tab:consistency}
\setlength{\tabcolsep}{4pt}
\begin{tabular}{lrrrr}
\toprule
& \multicolumn{2}{c}{Flip Rate} & \multicolumn{2}{c}{Quant Hurts}\\
\cmidrule(lr){2-3}\cmidrule(lr){4-5}
Dataset & Direct & HiFA4 & Direct & HiFA4\\
\midrule
ARC-C (295)      &7.46\%(22) &\textbf{1.69\%(5)} &11&\textbf{4}\\
TruthfulQA (817) &14.93\%(122)&\textbf{6.36\%(52)}&43&\textbf{12}\\
MMLU (14k)       &16.26\%(2283)&\textbf{8.16\%(1146)}&1071&\textbf{465}\\
HellaSwag (10k)  &9.02\%(906) &\textbf{4.75\%(477)}&516&\textbf{233}\\
\bottomrule
\end{tabular}
\end{table}

Table~\ref{tab:consistency} presents the main results on Qwen3-8B.
HiFA4 reduces the prediction flip rate substantially across all four
benchmarks: ARC-Challenge flips drop from 7.46\% to 1.69\%, and MMLU
flips fall from 16.26\% to 8.16\%.
Quant Hurts---samples that BF16 answers correctly but the quantized
model does not---are reduced by 57\% on MMLU (1071~to~465) and 55\%
on HellaSwag (516~to~233).

The sample-weighted accuracy loss relative to BF16 ($\delta_w$,
Table~\ref{tab:calib}) narrows from 1.12\,pp under direct HIF4
quantization to 0.70\,pp under HiFA4---recovering 37.5\% of the
quantization-induced degradation.
Notably, on TruthfulQA, HiFA4 actually \emph{exceeds} BF16 accuracy
(.725 vs.\ .704), which may reflect a mild regularization/noise effect
of K-outlier suppression on attention patterns
(Section~\ref{sec:sqk}) rather than a guaranteed improvement over BF16.

These gains are achieved without any model retraining, fine-tuning, or
modifications to the surrounding inference pipeline.
The only change from a standard BF16 FA execution is the substitution
of the C4V16-Aux operator described in Section~\ref{sec:method}.

\subsection{Decision Margin on Correctly-Answered Samples}

Beyond whether the final prediction is correct or incorrect, the
top-1/top-2 logit \emph{margin} of that prediction is meaningful: a model
that produces correct answers with a thin margin between the top two
logits is more likely to flip under slight distribution shift than one
with a large margin.
We treat this margin as a robustness proxy rather than a calibrated
confidence.
To probe this effect, we examine the top-1 minus top-2 logit gap across
samples that \emph{all three configurations} (BF16, Direct HIF4, HiFA4)
answer correctly, thereby controlling for correctness and isolating the
effect of quantization on the decision margin.

\begin{figure}[htbp]
\centering
\begin{subfigure}[b]{0.48\linewidth}
\includegraphics[width=\linewidth]{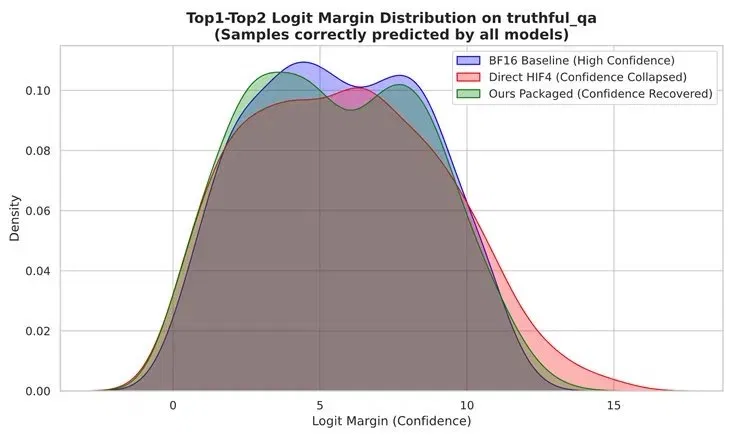}
\caption{TruthfulQA}
\end{subfigure}
\begin{subfigure}[b]{0.48\linewidth}
\includegraphics[width=\linewidth]{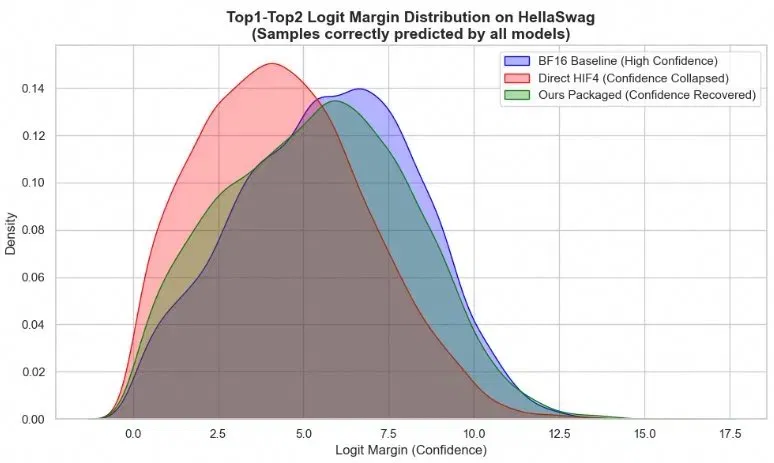}
\caption{HellaSwag}
\end{subfigure}
\caption{Top1--Top2 logit margin distribution, restricted to samples
all three configurations (BF16, Direct HIF4, HiFA4) answer correctly.
The distribution is a proxy for decision-margin robustness; a shift
leftward indicates reduced robustness even on correct predictions.
(In-panel labels use ``confidence'' loosely as a synonym for this
top-1/top-2 margin; we use ``decision margin'' in the text.)}
\label{fig:logit}
\end{figure}

Figure~\ref{fig:logit} shows that direct HIF4 quantization (red)
shifts the logit margin distribution noticeably leftward relative to
BF16 (blue) on both TruthfulQA and HellaSwag.
This shift occurs even for samples that all three models answer
correctly, indicating that quantization narrows the decision margin
in ways that are invisible to accuracy metrics but may affect
robustness under slight perturbation.
HiFA4 (green) recovers most of this margin: on TruthfulQA, the HiFA4
distribution nearly coincides with BF16, while on HellaSwag the
recovery is partial but substantial.
This result is consistent with the flip-rate improvements in
Table~\ref{tab:consistency}: suppressing K-outliers and correcting the
normalizer mismatch not only prevents more prediction changes but also
widens the decision margin of the remaining correct predictions.

\subsection{Long-Context Behavior}
\label{sec:longctx}

\begin{table}[htbp]
\centering\small
\caption{Needle-in-a-haystack results (Qwen3-8B, max 40K tokens).
Retrieval is saturated at 100\% and not discriminative.
The attention-output relative error
$\|O_{\mathrm{quant}}-O_{\mathrm{BF16}}\|_2/\|O_{\mathrm{BF16}}\|_2$
(prefill, averaged over all query positions and all layers) reveals
that Direct HIF4 accumulates error with context length; HiFA4 remains
stable.}
\label{tab:needle}
\begin{tabular}{lrrr}
\toprule
& BF16 & Direct HIF4 & HiFA4\\
\midrule
4K / 16K / 40K retrieval & 100\% & 100\% & 100\%\\
\midrule
4K attn.\ rel.\ $L_2$  & 0\% & 51.5\% & 12.5\%\\
16K attn.\ rel.\ $L_2$ & 0\% & 55.0\% & 12.8\%\\
40K attn.\ rel.\ $L_2$ & 0\% & 63.1\% & $\sim$12\%\\
K cache max-abs  & --- & 14.0   & 0.87--1.75\\
\bottomrule
\end{tabular}
\end{table}

The attention-output relative error reported in Table~\ref{tab:needle}
is the relative $L_2$ distance
$\|O_{\mathrm{quant}}-O_{\mathrm{BF16}}\|_2/\|O_{\mathrm{BF16}}\|_2$
between the quantized and BF16 attention outputs, computed over the
prefill stage and averaged over all query positions and all layers.
Because retrieval accuracy saturates at 100\% and is not discriminative,
this metric is the operative signal: Direct HIF4 error grows from 51.5\%
at 4K to 63.1\% at 40K, while HiFA4 stays near 12\% across context
lengths.
We note two limitations of this experiment that we plan to address in the
future version: (i) it is reported only on Qwen3-8B, and (ii) the
needle task does not probe long-context \emph{reasoning}.
A cross-model evaluation on RULER and a LongBench subset, together with
KV-cache memory and decode-side latency, is deferred to that version.

\subsection{Scheduling Implication of P-Reordering}
\label{sec:latency}

HIF4 quantization reduces Cube-path cost to 25\% of BF16 but shifts
the critical-path bottleneck to the vector path, which now includes
softmax and normalizer reduction.
P-Reordering eliminates this separate normalizer reduction by fusing it
into the $PV$ Cube GEMM (Section~\ref{sec:preorder}), reducing
vector-path pressure.
A preliminary instruction-scheduling analysis projects a 35.4\% critical-path latency reduction relative to BF16;
details and the full per-stage breakdown are in
Appendix~\ref{app:latency_model}.
On-hardware validation is deferred to when the target NPU becomes
publicly available (Section~\ref{sec:limitations}).

\subsection{Cross-Model Results}
\label{sec:crossmodel}

\begin{table}[htbp]
\centering\small
\caption{Accuracy across five models, each evaluated under its
gate-assigned configuration (Qwen3-8B Strong; Gemma2-9B Mild $\alpha=0.25$;
LLaMA3.1-8B, Mistral-7B, Phi-4B Bypass).
MMLU and HellaSwag use the full 14{,}042- and 10{,}042-sample sets for all
five models; ARC and TQA use the full evaluation sets.
$\Delta_{\mathrm{avg}}$: simple mean of signed per-benchmark gaps from
BF16; $\Delta_{w}=\sum_b N_b(\mathrm{Acc}^{\mathrm{HiFA4}}_b-\mathrm{Acc}^{\mathrm{BF16}}_b)/\sum_b N_b$:
sample-weighted signed gap over the four benchmarks with weights $N_b$
equal to their evaluation-set sizes
($N_{\mathrm{MMLU}}\!=\!14{,}042$, $N_{\mathrm{Hella}}\!=\!10{,}042$,
$N_{\mathrm{TQA}}\!=\!817$, $N_{\mathrm{ARC}}\!=\!295$).
Both are signed (negative $=$ below BF16).}
\label{tab:cross}
\setlength{\tabcolsep}{3pt}
\begin{tabular}{llrrrrrr}
\toprule
Model (Gate) & Method & MMLU & ARC & TQA & Hella & $\Delta_{\mathrm{avg}}$ & $\Delta_{w}$\\
\midrule
\multirow{2}{*}{Qwen3-8B (Strong)}
  & BF16   &.714&.885&.704&.865& ---& ---\\
  & HiFA4  &.706&.885&.725&.857& $+0.1$\,pp& $-0.7$\,pp\\
\midrule
\multirow{2}{*}{Gemma2-9B (Mild)}
  & BF16   &.675&.892&.422&.603& ---& ---\\
  & HiFA4  &.672&.885&.453&.588& $+0.2$\,pp& $-0.7$\,pp\\
\midrule
\multirow{2}{*}{LLaMA3.1-8B (Bypass)}
  & BF16   &.610&.803&.438&.547& ---& ---\\
  & HiFA4  &.607&.773&.426&.519& $-1.8$\,pp& $-1.4$\,pp\\
\midrule
\multirow{2}{*}{Mistral-7B (Bypass)}
  & BF16   &.586&.800&.518&.706& ---& ---\\
  & HiFA4  &.588&.790&.523&.701& $-0.2$\,pp& $-0.1$\,pp\\
\midrule
\multirow{2}{*}{Phi-4B (Bypass)}
  & BF16   &.663&.837&.524&.726& ---& ---\\
  & HiFA4  &.656&.841&.503&.720& $-0.8$\,pp& $-0.7$\,pp\\
\bottomrule
\end{tabular}
\end{table}

\begin{table}[htbp]
\centering\small
\caption{Full-set decision metrics for the three bypass models
(Direct HIF4 vs.\ HiFA4) on MMLU (14{,}042) and HellaSwag (10{,}042).
In bypass mode Smooth-QK is disabled ($\hat{s}=1$), so the improvement
over Direct HIF4 isolates the joint effect of P-Reordering and the
Q-Mean auxiliary on models without a concentrated K-outlier.
Acc\,=\,accuracy; Fl\,=\,flips vs.\ BF16; Hu\,=\,Quant Hurts;
He\,=\,Quant Helps; $\Delta$Hu\,=\,Quant-Hurt reduction.
Across all three models HiFA4 reduces Quant Hurts by 39--52\% without
any K-smoothing.}
\label{tab:bypass_full}
\setlength{\tabcolsep}{3.5pt}
\begin{tabular}{lllrrrrr}
\toprule
Model & Dataset & Method & Acc & Fl & Hu & He & $\Delta$Hu\\
\midrule
\multirow{6}{*}{LLaMA3.1-8B}
 & \multirow{3}{*}{MMLU}  & BF16        & .6098 & ---  & ---  & ---  & ---\\
 &                        & Direct HIF4 & .5769 & 2073 & 976  & 514  & ---\\
 &                        & HiFA4       & \textbf{.6070} & \textbf{1279} & \textbf{468} & 429 & $-52.0\%$\\
\cmidrule(lr){2-8}
 & \multirow{3}{*}{HellaSwag} & BF16    & .5466 & ---  & ---  & ---  & ---\\
 &                        & Direct HIF4 & .4883 & 2539 & 1215 & 630  & ---\\
 &                        & HiFA4       & \textbf{.5190} & \textbf{1674} & \textbf{745} & 468 & $-38.7\%$\\
\midrule
\multirow{6}{*}{Mistral-7B}
 & \multirow{3}{*}{MMLU}  & BF16        & .5857 & ---  & ---  & ---  & ---\\
 &                        & Direct HIF4 & .5819 & 1957 & 696  & 642  & ---\\
 &                        & HiFA4       & \textbf{.5875} & \textbf{1282} & \textbf{408} & 433 & $-41.4\%$\\
\cmidrule(lr){2-8}
 & \multirow{3}{*}{HellaSwag} & BF16    & .7062 & ---  & ---  & ---  & ---\\
 &                        & Direct HIF4 & .6808 & 1213 & 614  & 359  & ---\\
 &                        & HiFA4       & \textbf{.7009} & \textbf{821}  & \textbf{355} & 301 & $-42.2\%$\\
\midrule
\multirow{6}{*}{Phi-4B}
 & \multirow{3}{*}{MMLU}  & BF16        & .6629 & ---  & ---  & ---  & ---\\
 &                        & Direct HIF4 & .6355 & 2113 & 943  & 558  & ---\\
 &                        & HiFA4       & \textbf{.6560} & \textbf{1246} & \textbf{493} & 397 & $-47.7\%$\\
\cmidrule(lr){2-8}
 & \multirow{3}{*}{HellaSwag} & BF16    & .7260 & ---  & ---  & ---  & ---\\
 &                        & Direct HIF4 & .7124 & 1374 & 632  & 496  & ---\\
 &                        & HiFA4       & \textbf{.7196} & \textbf{842}  & \textbf{369} & 305 & $-41.6\%$\\
\bottomrule
\end{tabular}
\end{table}

Table~\ref{tab:cross} shows aggregate accuracy under each model's
gate-assigned configuration.
On Qwen3-8B (Strong Smooth-QK), HiFA4 is within 0.7\,pp of BF16
($\Delta_w=-0.7$\,pp), with the simple mean essentially flat
($\Delta_{\mathrm{avg}}=+0.1$\,pp).
On Gemma2-9B (Mild Smooth-QK, $\alpha=0.25$), HiFA4 stays close to BF16
($\Delta_{\mathrm{avg}}=+0.2$\,pp, $\Delta_w=-0.7$\,pp); relative to the
Direct HIF4 baseline it improves decision consistency, reducing full-set
MMLU Quant Hurts from 445 to 323 ($-27\%$) and recovering accuracy from
66.75\% to 67.23\% (BF16 67.49\%), while on TruthfulQA it even exceeds
BF16 (.453 vs.\ .422).
These two models confirm that Smooth-QK is beneficial when the gate
detects a concentrated K-outlier ($\rho_K\geq6$).

For the three bypass models (LLaMA3.1-8B, Mistral-7B, Phi-4B), the gate
disables Smooth-QK, and the remaining components---P-Reordering and the
Q-Mean auxiliary---still substantially reduce decision drift relative to
Direct HIF4.
Table~\ref{tab:bypass_full} reports the full-set MMLU and HellaSwag
metrics: Quant Hurts fall by 39--52\% on every model/benchmark pair
without any K-smoothing (e.g.\ LLaMA3.1-8B MMLU $976\to468$, $-52\%$;
Phi-4B MMLU $943\to493$, $-48\%$; Mistral-7B HellaSwag $614\to355$,
$-42\%$).
This demonstrates that the P-Reordering\,$+$\,Q-Mean benefit is not
specific to models with a K-dominant outlier pattern, and is the
mechanism by which HiFA4 remains useful precisely where Smooth-QK does
not apply.

LLaMA3.1-8B is an instructive boundary case that motivates the
two-dimensional gate.
Its K-to-Q asymmetry is not the issue---$T_{K/Q}=2.70$ actually
\emph{exceeds} that of Gemma2-9B ($2.17$), which is smoothed---but its
intra-K outlier severity $\rho_K=4.78$ falls below the $\rho_K\geq6$
threshold, indicating K-outliers too diffuse to benefit from difficulty
transfer.
Consistent with this, applying Smooth-QK at any tested strength
($\alpha\in[0.05,0.35]$) fails to improve---and on the diagnostic
benchmarks slightly degrades---its decision consistency relative to the
bypass configuration (Appendix~\ref{app:smooth_dist}).
The gate therefore routes it to bypass, where P-Reordering and Q-Mean
alone deliver the largest MMLU Quant-Hurt reduction of any model
($-52\%$).
Had the gate used $T_{K/Q}$ alone, LLaMA3.1-8B would have been
incorrectly smoothed; this case shows that the decision is driven by the
\emph{structure} of the outlier ($\rho_K$), not merely the presence of
K--Q asymmetry.

\section{Ablation}
\label{sec:ablation}

\begin{table*}[htbp]
\centering\small
\caption{Stepwise ablation on Qwen3-8B.
Row~A is Direct HIF4 without any corrections.
Rows~B--E include the Q-Mean auxiliary from~\citet{zhang2025sage2}.
MMLU and HellaSwag use full evaluation sets.
ARC and TruthfulQA values in row~E match Table~\ref{tab:consistency}.
Bold\,=\,best per column.}
\label{tab:ablation}
\begin{tabular}{lcc rr rr rr rr}
\toprule
Config & SQK & PRe
  & \multicolumn{2}{c}{MMLU (14k)}
  & \multicolumn{2}{c}{HellaSwag (10k)}
  & \multicolumn{2}{c}{ARC (full)}
  & \multicolumn{2}{c}{TruthfulQA (full)}\\
\cmidrule(lr){4-5}\cmidrule(lr){6-7}\cmidrule(lr){8-9}\cmidrule(lr){10-11}
&&& Flip & Hurt & Flip & Hurt & Flip & Hurt & Flip & Hurt\\
\midrule
A.\ Direct HIF4    &---&---& 2283&1071& 906&516& 22&11&122&43\\
B.\ +Q-Mean        &---&---& 1595& 693& 601&319& 13& 6& 89&29\\
C.\ +Smooth-QK     &\checkmark&---& 1279& 536& 541&301& 11& 5& \textbf{50}&\textbf{11}\\
D.\ +P-Reordering  &---&\checkmark& 1476& 601& 553&\textbf{271}& 13& 7& 87&29\\
E.\ Both (HiFA4)   &\checkmark&\checkmark& \textbf{1146}&\textbf{465}&\textbf{477}&\textbf{233}&\textbf{5}&\textbf{4}& 52&\textbf{12}\\
\bottomrule
\end{tabular}
\end{table*}

\textbf{Smooth-QK (C vs.\ B):}
MMLU flips $-20\%$, hurts $-23\%$.
Strongest on TruthfulQA (flips $89\to50$, hurts $29\to11$), consistent
with K-outlier suppression having the largest effect on semantically
sensitive decisions.

\textbf{P-Reordering (D vs.\ B):}
Achieves the lowest HellaSwag Quant Hurts of any single mechanism
($271$ vs.\ $301$ for Smooth-QK alone), consistent with the normalizer
error accumulating proportionally with the number of active sequence
positions.
MMLU hurts $-13\%$.
We note that this comparison isolates P-Reordering relative to the
Q-Mean-only baseline (row B), not relative to an unquantized-V control;
the latter is the subject of the planned $P$-only local-operator study
discussed in Section~\ref{sec:preorder}.

\textbf{Full HiFA4 (E):}
Outperforms either mechanism alone on every column.
ARC flips drop to 5, below C ($=11$) and D ($=13$), confirming
the two mechanisms address complementary failure modes.

\section{Discussion}

\paragraph{Why the coherent error is hard to detect.}
Standard distance metrics such as relative $L_2$ are dominated by
the independent per-dimension noise from Q, K, V quantization.
The coherent term $\bar\varepsilon O^*$ changes relative $L_2$ only
weakly, yet has a consistent sign ($\bar\varepsilon<0$ in the measured
trace) and its aggregate effect across layers shifts logit ordering.
PFR and Quant Hurts are sensitive to this accumulated effect.

\paragraph{Bypass models still benefit.}
On the three bypass models (LLaMA3.1-8B, Mistral-7B, Phi-4B), Smooth-QK
is disabled, yet HiFA4 reduces full-set Quant Hurts by 39--52\%
(Table~\ref{tab:bypass_full}).
This shows that P-Reordering and Q-Mean together address failure
modes that are not specific to the concentrated-K-outlier pattern, and
that the gate's decision to bypass these models loses no decision-level
accuracy.

\paragraph{Limitations.}
\label{sec:limitations}
The most significant limitation is that all latency numbers are
theoretical instruction-scheduling estimates; the target Ascend NPU is
not yet publicly available, and on-hardware throughput and latency
measurements will be reported when it is.

The applicability thresholds ($\rho_K\geq6$ and $T_{K/Q}\geq2$ to enable
smoothing; $T_{K/Q}>3$ for Strong vs.\ Mild) are empirical boundaries
fitted to five models.
With only five points, these thresholds have limited statistical support:
the $\rho_K\geq6$ boundary falls in the gap between LLaMA3.1-8B
($\rho_K=4.78$, bypassed) and Gemma2-9B ($\rho_K=8.11$, smoothed), and a
larger model suite would be needed to locate it precisely or to confirm
its functional form. We therefore present the gate as a calibrated
heuristic that is consistent with our per-model ablations
(Section~\ref{sec:crossmodel}, Appendix~\ref{app:smooth_dist}), not as a
universally validated criterion.

Smooth-QK itself is effective only on models whose K activations carry a
concentrated, parameter-induced outlier; among our five models this holds
for Qwen3-8B (the QK-RMSNorm channel-51 pattern) and, more mildly, for
Gemma2-9B. On models without such structure (LLaMA3.1-8B, Mistral-7B,
Phi-4B) smoothing is correctly bypassed, and HiFA4's benefit comes
entirely from P-Reordering and Q-Mean. Characterizing how prevalent the
concentrated-K-outlier pattern is across the broader model population is
left to future work; we do not claim Smooth-QK is broadly applicable
beyond models exhibiting it.

The end-to-end isolation of P-Reordering relies on the row B vs.\ D
ablation, in which Q-Mean is already active and V is quantized; a
cleaner $P$-only local-operator ablation (Q, K, V held at BF16, varying
only the normalizer denominator, reported via cosine similarity and
output-norm ratio) is deferred to a future version, as is a
cross-layer and cross-model characterization of $\bar\varepsilon$.

The long-context evaluation (Section~\ref{sec:longctx}) currently uses a
retrieval task that saturates at 100\% accuracy and is reported on a
single model; RULER/LongBench coverage, KV-cache memory, and decode-side
latency are deferred.

\section{Conclusion}

HiFA4 quantizes both $QK^T$ and $PV$ in FlashAttention to HIF4 4-bit
for LLM inference on Ascend NPUs via two mechanisms.
Smooth-QK traces K-activation outliers to fixed model parameters and
eliminates them with a calibration-static per-channel rescaling
requiring no per-tile online reduction and only a fusible element-wise
scaling.
P-Reordering proves and eliminates a coherent output error arising
from normalizer--GEMM inconsistency (Theorem~\ref{thm:main}),
validated across 3.6~million attention tiles, and additionally
fuses the normalizer into the $PV$ Cube GEMM.
On Qwen3-8B: MMLU Quant Hurts $1071\to465$, ARC flips $22\to5$,
sample-weighted $\delta_w$ narrowed from 1.12\,pp to 0.70\,pp
(Table~\ref{tab:calib}).
On Gemma2-9B (mild smoothing): MMLU Quant Hurts $445\to323$, within
0.7\,pp of BF16.
On the three bypass models, P-Reordering and Q-Mean alone reduce full-set
MMLU Quant Hurts by 41--52\% (LLaMA3.1-8B $976\to468$; Phi-4B
$943\to493$; Mistral-7B $696\to408$),
demonstrating that this component generalizes beyond models
with concentrated K-outliers.
A theoretical latency model projects a 35.4\% reduction relative to BF16,
to be validated on hardware when the target NPU becomes available.

\bibliographystyle{plainnat}
\bibliography{refs}

\appendix

\section{Caveats on the FP4 Attention Port}
\label{app:sa3caveat}

The motivation experiment discussed in the introduction applies the published FP4
attention recipe of~\citet{zhang2025sage3} to Qwen3-8B.
This is an algorithmic port, not a controlled ablation or a comparison
of optimized kernels.
At least four factors differ simultaneously between this port and HiFA4:

\textbf{(1)~Hardware and quantization format.}
SageAttention3 targets NVIDIA Blackwell's NVFP4, a 16-element-group
block-floating-point format with an FP8-E4M3 scale.
HiFA4 targets Ascend HIF4's three-level hierarchical format with a
64-element group, a global E6M2 base scale, two levels of 1-bit
micro-exponents, and S1P2 elements.
The two formats respond differently to the same activation distributions.

\textbf{(2)~Smoothing calibration.}
SageAttention3's smoothing parameters are tuned for diffusion model
activations (CogVideoX, HunyuanVideo, etc.).
The port applies those parameters to Qwen3-8B, whose activations
have a qualitatively different outlier structure (QK-RMSNorm-induced,
parameter-level, head-invariant) not present in diffusion models.

\textbf{(3)~$P$-quantization granularity.}
SageAttention3 quantizes attention weights at a different granularity
than HIF4's 64-element grouping.

\textbf{(4)~Normalizer path.}
SageAttention3 accumulates the softmax normalizer from a
higher-precision reconstruction of $\hat{P}$ (see Section~\ref{sec:preorder}).
This is the mismatch that P-Reordering is designed to correct.

The degradation observed in this port reflects all four factors acting
simultaneously.
It demonstrates that the published recipe does not transfer directly
to a language model on different hardware, not that
the authors' optimized kernel fails on its intended workloads.
We present it as a \emph{portability stress test}, not a comparison.

\section{Experimental Settings}
\label{app:setup}

This appendix documents the full evaluation configuration to support
reproducibility.

\paragraph{Scoring protocol.}
All benchmarks are evaluated as multiple-choice tasks scored by
comparing the per-option likelihoods of the candidate continuations; we
do not use free-form generation.
The predicted label is the argmax over option likelihoods, and the
top-1/top-2 logit margin in Figure~\ref{fig:logit} is computed on these
option logits.
ARC-Challenge, TruthfulQA (MC1), GPQA-Diamond, MMLU, and HellaSwag all
follow this protocol.

\paragraph{Shot count and prompt template.}
Evaluation is zero-shot with the standard task templates; the same
template string is used for BF16, Direct HIF4, and HiFA4 so that the
only varying factor is the numerical precision of attention.
Answer extraction is by option-likelihood ranking as above and does not
depend on string parsing.

\paragraph{Precision and dtype.}
The \textbf{BF16 baseline} (labeled \textbf{C16V16} in
Table~\ref{tab:latency}) is the unmodified FlashAttention execution in
which both the $QK^T$ and $PV$ GEMMs use BF16 operands with FP16 online
softmax state---the exact floating-point operator baseline against which
all quantized variants are compared.
For the quantized configurations, only the attention $QK^T$ and $PV$
GEMMs are computed in HIF4 (with the online softmax state, normalizer,
and final normalization in FP16 as detailed in Appendix~\ref{app:impl});
all other model weights and activations remain in BF16.
This isolates the effect of attention quantization.

\paragraph{Calibration provenance and leakage.}
Smooth-QK calibration uses 1024 samples drawn exclusively from training
or auxiliary splits, never from any evaluation split, so there is no
calibration/test leakage.
Scale factors are computed once offline and held fixed for all
evaluation runs.

\paragraph{Hook coverage.}
All attention computation is routed through the C4V16 dispatch hook; no
attention call falls back to a BF16 path.
We log the number of hook invocations per evaluation to verify complete
coverage.
For example, on Phi-4B the hook records 321{,}344 calls on the full
HellaSwag set and 449{,}344 calls on the full MMLU set, matching the
expected (layers~$\times$ samples~$\times$ tiles) counts for those runs.
Equivalent coverage logs are produced for every model/benchmark pair.

\paragraph{Software and hardware.}
Functional QDQ evaluation uses the official HIF4 reference
implementation following the BF16$\to$HIF4 conversion of Algorithm~1
in the HIF4 specification~\citep{huawei2026hif4} (64-element group,
three-level scaling hierarchy).
We will release the exact library, framework, and toolkit versions
(PyTorch / transformers / Ascend toolkit) alongside the code; on-hardware
timing awaits public availability of the target NPU.

\paragraph{Determinism.}
Evaluation runs are deterministic given a fixed model, fixed calibration
scales, and greedy (argmax) option scoring; no sampling is used, so the
reported decision-consistency counts are reproducible across runs.

\section{Implementation Details}
\label{app:impl}

\paragraph{Hook insertion point.}
All quantization is applied inside the FlashAttention dispatch hook,
after RoPE but before any GEMM.
Neither the projection layers nor the RoPE rotation itself is modified.

\paragraph{Smooth-QK placement.}
Smooth-QK must be applied after RoPE so that the rotation sees the
original per-head Q and K tensors.
Applying it before RoPE would require commutativity with the complex
exponentials in the rotation, which does not hold in general.

\paragraph{Quantized tensors.}
The following tensors are quantized to HIF4 before entering a Cube GEMM:
\begin{itemize}
\item $Q_z = \tilde{Q} - \mu$ (zero-mean query, post-Smooth-QK);
\item $\tilde{K} = K\,\mathrm{diag}(\hat{s})^{-1}$ (smoothed key);
\item $V$ (value, unchanged);
\item $\hat{P}_{ij}$ (per-tile softmax weights, before $PV$ GEMM).
\end{itemize}

\paragraph{FP16 operations.}
The following are computed and maintained in FP16 throughout:
\begin{itemize}
\item Running maximum $m_i$ (online softmax);
\item Softmax exponential $\exp(S_{ij}-m_i)$;
\item Normalizer $\ell_i$ (accumulated from $\hat{P}$, result in FP16);
\item Final normalization $O_i \leftarrow O_i/\ell_i$;
\item Q-Mean auxiliary GEMV (see below).
\end{itemize}

\paragraph{Q-Mean GEMV and coordinate consistency.}
The Q-Mean correction subtracts the per-token scalar mean of $\tilde{Q}$
over the head-feature (channel) dimension $d$:
$\mu_i = \frac{1}{d}\sum_{c=1}^{d}\tilde{Q}_{i,c}$,
which is a scalar per query token (not a mean across heads).
The correction compensates by adding $\mu_i\sum_{c}\tilde{K}_{j,c}$
to each tile's pre-softmax scores, i.e.,
$S_{ij}\mathrel{+}=\mu_i\mathbf{1}_{T_j}^T\tilde{K}_j$
where $\mathbf{1}_{T_j}$ sums all key-feature channels.
We use $\tilde{K}$ (post-Smooth-QK FP16 tensor) rather than the original
$K$ to maintain a consistent coordinate frame: since $\tilde{Q}$ is the
quantity fed into the main QK GEMM, the Q-mean should be referenced to
the same transformed space, and the correction should use the
correspondingly transformed key.
The $\tilde{K}$ tensor is available as a transient FP16 activation before
HIF4 packing and requires no additional memory allocation.

\paragraph{GQA and batch handling.}
In GQA configurations (e.g.\ Qwen3-8B uses 32 Q heads and 8 KV heads),
the scale factor $\hat{s}$ is computed per key-head channel and broadcast
across the corresponding Q heads.
This correctly handles the head-invariant nature of the outlier
(Section~\ref{sec:topology}): a single $\gamma_{K,c}$ parameter produces
the same activation at every Q head after GQA expansion.
Batch size and sequence length do not affect the algorithm;
the per-channel $\hat{s}$ is applied element-wise to the token dimension.

\section{Decision-Consistency Metric Definitions}
\label{app:metrics}

\paragraph{Why not aggregate accuracy?}
Aggregate accuracy changes can mask opposing effects.
A quantized model may corrupt some BF16-correct answers (Quant Hurts)
while accidentally fixing some BF16-incorrect answers (Quant Helps).
A model that hurts many samples but helps just as many would show zero
accuracy change while exhibiting substantial decision drift.
We use metrics that separate these effects.

\paragraph{Prediction Flip Rate (PFR).}
\begin{equation}
\text{PFR} = \frac{1}{N}\sum_{i=1}^N
  \mathbf{1}\!\left[\hat{y}_i^q \neq \hat{y}_i^{\text{bf16}}\right]
\end{equation}
PFR measures the fraction of evaluation samples for which the quantized
model's top-1 prediction differs from the BF16 model's prediction.
A BF16 model has PFR\,=\,0 by definition.
PFR is related to but not equal to accuracy change: a flip may move
from a correct BF16 prediction to an incorrect one (a Hurt), or from
an incorrect BF16 prediction to a different incorrect one.

\paragraph{Quant Hurts.}
\begin{equation}
\text{Hurts} = \left|\left\{i :
  \hat{y}_i^{\text{bf16}} = y_i \;\text{and}\;
  \hat{y}_i^q \neq y_i\right\}\right|
\end{equation}
Quant Hurts counts only the regressions: samples the BF16 model answers
correctly that the quantized model does not.
We report Hurts as a count (rather than a rate) because it is bounded
above by the BF16 model's accuracy and is most interpretable as a raw
count of corrupted inferences.

\paragraph{Relation to accuracy.}
Defining Helps\,$=|\{i:\hat{y}_i^{\text{bf16}}\neq y_i,\,\hat{y}_i^q=y_i\}|$,
the accuracy change satisfies $\Delta\text{Acc}=(\text{Helps}-\text{Hurts})/N$.
We emphasize Hurts over $\Delta\text{Acc}$ because Helps often reflect
quantization noise accidentally realigning a borderline wrong answer,
whereas Hurts identify samples where the model has lost a reliable capability.

\paragraph{BF16 as reference.}
All metrics compare against the same BF16 model on identical inputs.
This controls for prompt format, tokenization, and sampling---only
the numerical precision of the attention computation differs.

\section{Outlier Channel Stability (Jaccard)}
\label{app:jaccard}

A central assumption behind calibration-static Smooth-QK deployment is
that the \emph{identity} of the outlier channels---not just their
magnitude---is stable across different inputs.
If a channel is a dominant outlier for one prompt but not another, a
static scale factor computed from calibration data would suppress the
wrong channels at inference time, potentially harming accuracy.

To test this assumption, we run each of our five evaluation models on
eight diverse prompts spanning code generation, mathematical reasoning,
factual question answering, multi-turn dialogue, and narrative text.
For each model and prompt, we collect the top-$N$ channels by maximum
activation magnitude at Layer~0 (where $N = 5\%$ of the total channel
count).
We then compute the pairwise Jaccard similarity between all 28 pairs of
prompt-specific channel sets---where a Jaccard similarity of~1.0 would
indicate the identical set of outlier channels across both prompts, and
0 would indicate no overlap.
High mean off-diagonal Jaccard ($\gg 0.5$) across diverse prompts
indicates that the outlier channel identity is an intrinsic model
property rather than an input artifact.

\begin{figure}[htbp]
\centering
\begin{subfigure}[b]{0.36\linewidth}
\includegraphics[width=\linewidth]{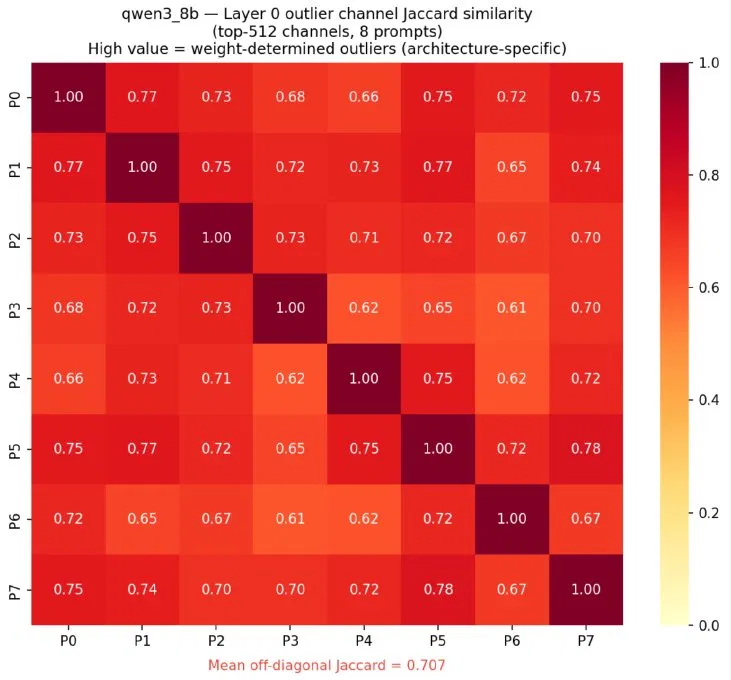}
\caption{Qwen3-8B (0.707)}
\end{subfigure}
\begin{subfigure}[b]{0.36\linewidth}
\includegraphics[width=\linewidth]{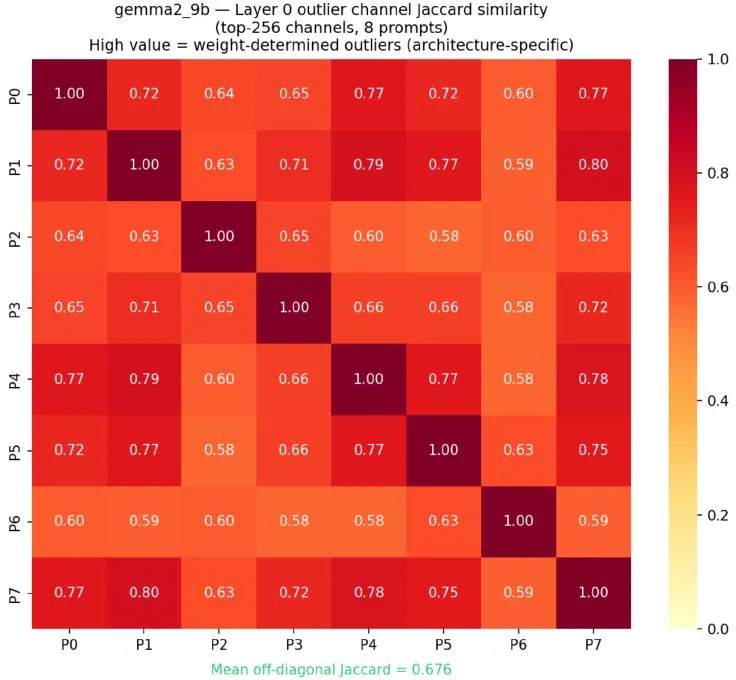}
\caption{Gemma2-9B (0.676)}
\end{subfigure}
\begin{subfigure}[b]{0.36\linewidth}
\includegraphics[width=\linewidth]{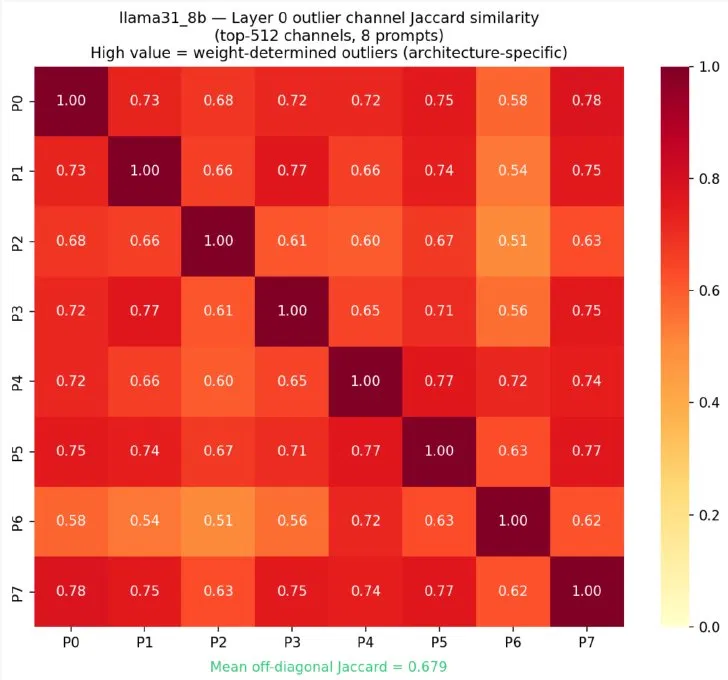}
\caption{LLaMA3.1-8B (0.679)}
\end{subfigure}
\begin{subfigure}[b]{0.36\linewidth}
\includegraphics[width=\linewidth]{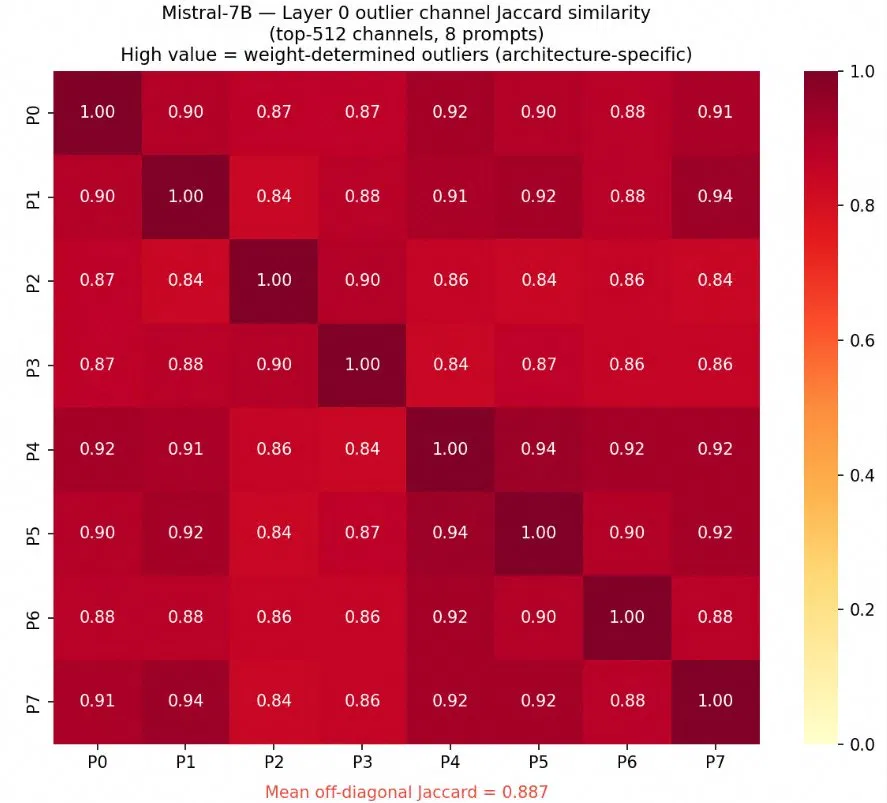}
\caption{Mistral-7B (0.887)}
\end{subfigure}
\caption{Jaccard similarity of top-$N$ outlier channels across 8 prompts.
Mean off-diagonal values of 0.676--0.887 confirm that the same channels
are identified regardless of input, supporting static calibration.
Phi-4B Jaccard is 0.786 (Table~\ref{tab:outlier}).
High Jaccard in Mistral-7B and Phi-4B does not trigger smoothing;
$\rho_K<6$ in both cases routes them to bypass.}
\label{fig:jaccard}
\end{figure}

\section{Smooth-QK Scale Factor by Layer}
\label{app:smooth_dist}

The Smooth-QK calibration procedure computes per-layer, per-channel
scale factors $\hat{s}$ for every model, regardless of the gate
outcome.
Table~\ref{tab:smooth_layers} reports the resulting statistics for
Qwen3-8B, where the correction is most pronounced.
All statistics in this table are RMS-aggregated over the 1024-sample
calibration set (the statistic that drives Eq.~\eqref{eq:sqk}), and are
therefore not directly comparable to the per-slice (single batch element,
single head) max-abs values shown in Figure~\ref{fig:smooth_effect}.

\begin{table}[h]
\centering\small
\caption{K outlier statistics and resulting Smooth-QK scale factor
for Qwen3-8B, grouped by layer depth (calibration-set RMS aggregate).
Layer~0 dominates; deeper layers have near-unity $\hat{s}$ and are
effectively unaffected by smoothing.
Q Expand gives the resulting range of $\max|\tilde{q}_c|/\max|q_c|$
across channels.}
\label{tab:smooth_layers}
\begin{tabular}{lrrrcc}
\toprule
Layer & MaxAbs & P99 & $\rho_K$ & $\hat{s}$ & Q Expand\\
\midrule
0     & 220.8 & 16.9 & 13.1 & 3.72 & 0.56--1.40$\times$\\
1--4  & 10--25& 8--16& 1.3--1.7& 1.1--1.3 & 0.8--1.2$\times$\\
5--35 & 12--35& 9--19& $\sim$1.8& $\sim$1.1& 0.9--1.1$\times$\\
\bottomrule
\end{tabular}
\end{table}

\paragraph{On the different magnitudes reported for the same outlier.}
Three numerical values describing the channel-51 outlier appear in the
paper, and they are deliberately \emph{different measurements} rather than
the same quantity computed three ways.
(i) Table~\ref{tab:smooth_layers} reports MaxAbs $=220.8$ at Layer~0:
this is the calibration-set RMS aggregate of the per-channel maximum and
is the statistic that drives Eq.~\eqref{eq:sqk}; the corresponding
$\hat{s}\!=\!3.72$ would compress it to $\approx$59 ($=220.8/3.72$) under
that scale.
(ii) Figure~\ref{fig:smooth_effect} reports max-abs $=104.5$ (before) and
$=2.516$ (after) for a single fixed slice (one batch element, head~0,
first 64 tokens), where ``after'' is measured directly on the
Smooth-QK-transformed K tensor.
(iii) Figure~\ref{fig:causal} reports $\approx$155 from a single-token
forward trace on a different input, illustrating the per-token
amplification mechanism.
These values differ because the outlier magnitude varies across heads,
tokens, and inputs, and because the three measurements use different
aggregation levels (single-token trace vs.\ single-slice max-abs vs.\
calibration-set RMS aggregate) and were produced by different scripts
(the main inference hook for the figures, a separate calibration script
for Table~\ref{tab:smooth_layers}).
They are therefore not expected to coincide; we annotate each figure and
table with its statistic so the reader can interpret them correctly.
In particular, the after-value $2.516$ in Figure~\ref{fig:smooth_effect}
is the measured result of applying Smooth-QK to that slice and should not
be cross-checked against $220.8/3.72$ or against the
$|K_{\max}|/|Q_{\max}|$ ratio of any other statistic, since it pertains
to a different head and input.

\paragraph{Cross-model comparison.}
The concentration pattern differs substantially across architectures.

\textbf{Qwen3-8B (strong gate):} Layer~0 is the dominant outlier layer
with $\rho_K=13.1$ and $T_{K/Q}=14.5$.
The scale factor at Layer~0 is 3.72, compressing K's maximum from
220.8 to approximately 59 and expanding Q proportionally.
All other layers have near-unity $\hat{s}$ ($\leq1.3$).
This concentrated pattern is consistent with the QK-RMSNorm affine
weight $\gamma_{K,51}\approx34$ being the root cause: it acts only
at the layer where it is defined.

\textbf{Gemma2-9B (mild gate, $\alpha=0.25$):} Layer~0 has $\rho_K=8.1$
(above the $\rho_K\geq6$ threshold) but $T_{K/Q}=2.17$.
Because K does not strongly dominate Q, the gate selects mild smoothing
($\alpha=0.25$) rather than $\alpha=0.5$, transferring less difficulty to
Q. An $\alpha$ sweep on the diagnostic suite confirmed $\alpha=0.25$ as
the best of $\{0.05,0.1,0.25,0.35\}$; relative to Direct HIF4 it lowers
full-set MMLU and HellaSwag Quant Hurts ($445\to323$ and $449\to406$).
The outlier is present but less extreme than in Qwen3-8B.

\textbf{LLaMA3.1-8B (bypass):} $\rho_K$ peaks at only 4.78, below the
$\rho_K\geq6$ smoothing threshold, and there is no single dominant outlier
layer; the outlier ratio is spread more evenly across layers.
Although its $\rho_K$ is the highest among the bypass models, an $\alpha$
sweep ($\alpha\in\{0.05,0.1,0.2,0.25,0.35\}$) shows that enabling
Smooth-QK does not improve, and on some diagnostic benchmarks slightly
degrades, its decision consistency relative to the bypass configuration.
The gate therefore routes LLaMA3.1-8B to bypass, consistent with the
empirical result; this is the boundary case discussed in
Section~\ref{sec:crossmodel}.

\textbf{Mistral-7B and Phi-4B (bypass):}
Both models have low intra-K outlier severity ($\rho_K=3.56$ and $3.69$,
below the $\rho_K\geq6$ threshold) together with $T_{K/Q}<2$.
For Phi-4B, the max-layer $\rho_K=3.69$ occurs at Layer~30,
and the max-layer $T_{K/Q}=1.50$ is at a different layer.
At Layer~0, Phi-4B actually has $T_{K/Q}=0.56$, meaning
\emph{K is smaller than Q}---the reverse of the Qwen3-8B pattern.
This confirms that the outlier topology can differ fundamentally across
architectures, and the gate correctly identifies that no smoothing
should be applied.

\paragraph{Why the calibration procedure covers all layers.}
Even though only Layer~0 of Qwen3-8B has a large $\hat{s}$, the
calibration procedure computes $\hat{s}_c$ for every layer and channel
of every model.
The gate is then applied independently per layer.
For models in bypass mode, all $\hat{s}_c$ values are set to 1.0 at
inference, effectively a no-op.
This design ensures that if future models have outliers at deeper layers,
no algorithmic change is needed: the calibration will find them and the
gate will route them appropriately.

\section{Cross-Model Ablation (Diagnostic Suite)}
\label{app:cross_ablation}

\begin{table}[htbp]
\centering\small
\caption{Stepwise ablation on Gemma2-9B (diagnostic benchmarks;
Mild Smooth-QK at $\alpha=0.25$).
Row~A is Direct HIF4; rows~B--E add the Q-Mean auxiliary.
Rows C and E apply Smooth-QK; rows A, B, D do not and are
$\alpha$-independent.
Fl\,=\,Flip count, Hu\,=\,Quant Hurt count.
The full pipeline (E) is best on TQA and GPQA; on the small ARC set
(295 samples) the SQK-only row C is marginally lower.}
\label{tab:gemma_ablation}
\setlength{\tabcolsep}{3.5pt}
\begin{tabular}{lcc rr rr rr}
\toprule
Config & SQK & PRe
  & \multicolumn{2}{c}{ARC (295)} & \multicolumn{2}{c}{TQA (817)}
  & \multicolumn{2}{c}{GPQA (198)}\\
\cmidrule(lr){4-5}\cmidrule(lr){6-7}\cmidrule(lr){8-9}
&&& Fl & Hu & Fl & Hu & Fl & Hu\\
\midrule
A.\ Direct HIF4    &---&---& 13&7& 99&25& 49&12\\
B.\ +Q-Mean        &---&---&  8&4& 87&23& 35&12\\
C.\ +Smooth-QK     &\checkmark&---& \textbf{4}&\textbf{1}& 70&17& 37&10\\
D.\ +P-Reordering  &---&\checkmark& 9&5& 78&16& 37&14\\
E.\ Both (HiFA4)   &\checkmark&\checkmark& 10&5& \textbf{74}&\textbf{11}& \textbf{33}&\textbf{8}\\
\bottomrule
\end{tabular}
\end{table}

The Gemma2-9B stepwise result corroborates the Qwen3-8B findings.
Relative to Direct HIF4 (row~A), the full HiFA4 configuration (row~E)
reduces Quant Hurts on every diagnostic benchmark, with the largest
effect on TruthfulQA (hurts $25\to11$, a $56\%$ reduction), consistent
with semantically sensitive decisions benefiting most.
Smooth-QK (row~C vs.\ B) and P-Reordering (row~D vs.\ B) each contribute:
Smooth-QK is strongest on ARC and TQA, while P-Reordering also lowers TQA
hurts independently of smoothing.
Combining them (row~E) is best on TQA and GPQA; on ARC, where only 295
samples make the counts noisy, the SQK-only configuration is marginally
lower, but the difference (Hu 1 vs.\ 5) is within run-to-run variation on
a set this size.

\section{Bypass-Model Full-Set Provenance}
\label{app:phi_ablation}

Table~\ref{tab:bypass_full} in the main text reports the three bypass
models (LLaMA3.1-8B, Mistral-7B, Phi-4B) on the full MMLU and HellaSwag
sets, comparing the Direct HIF4 baseline (no corrections) against the
bypass HiFA4 configuration (Smooth-QK disabled, $\hat{s}=1$; P-Reordering
and Q-Mean active).
Because Smooth-QK is off, the Direct$\to$HiFA4 improvement isolates the
joint effect of P-Reordering and the Q-Mean auxiliary on models without a
concentrated K-outlier.
All three models show Quant-Hurt reductions of 39--52\% on both datasets
(LLaMA3.1-8B MMLU $976\to468$, HellaSwag $1215\to745$; Mistral-7B MMLU
$696\to408$, HellaSwag $614\to355$; Phi-4B MMLU $943\to493$, HellaSwag
$632\to369$).
The hook-coverage logs for these runs (e.g.\ Phi-4B HellaSwag 321{,}344
calls, MMLU 449{,}344 calls) confirm that all attention computation was
routed through the C4V16 operator (Appendix~\ref{app:setup}).

\section{Ablation Sub-Result Provenance}
\label{app:ablation_status}

Table~\ref{tab:ablation} in Section~\ref{sec:ablation} reports a
stepwise component ablation on Qwen3-8B.
This appendix documents the evaluation setup for each row to ensure
reproducibility.

\paragraph{MMLU (14{,}042 samples).}
All five configurations (A--E) were evaluated on the full MMLU dataset.
Each run uses the complete standard MMLU evaluation split
(14{,}042 questions across 57 subjects).
Rows A and E match the values reported in Tables~\ref{tab:consistency}
and~\ref{tab:cross} for Direct HIF4 and HiFA4, respectively.

\paragraph{HellaSwag (10{,}042 samples).}
All five configurations were evaluated on the full HellaSwag
validation set.
Rows A and E match the corresponding values in Table~\ref{tab:consistency}.

\paragraph{ARC-Challenge (295 samples) and TruthfulQA (817 samples).}
All five configurations were evaluated on the complete ARC-Challenge
and TruthfulQA evaluation sets.
Row~E values match Table~\ref{tab:consistency} exactly.

\paragraph{Why Q-Mean (row B) is held fixed.}
The goal of the ablation is to isolate the contributions of
Smooth-QK and P-Reordering, which are the two mechanisms proposed
in this paper.
Q-Mean is adopted from~\citet{zhang2025sage2} and is not a contribution
of this work.
Including it as a fixed component in rows B--E ensures that the
incremental differences between rows reflect only our proposed
mechanisms.
Row~A (Direct HIF4, no corrections) provides the absolute baseline.
Row~B (+Q-Mean only) provides the baseline after applying the
inherited component, against which our mechanisms are compared.

\paragraph{Gemma2-9B stepwise component effect.}
The stepwise component effect for Gemma2-9B (diagnostic suite: ARC,
TruthfulQA, GPQA) is reported in Table~\ref{tab:gemma_ablation}
in Appendix~\ref{app:cross_ablation}.

\section{Instruction-Scheduling Latency Model}
\label{app:latency_model}

This appendix reports the quantitative instruction-scheduling analysis
summarized in Section~\ref{sec:latency}.
All figures are theoretical estimates derived from an Ascend NPU
cycle-count model at $T=40{,}960$ tokens; \textbf{no on-hardware
measurements are available at this time} and on-hardware validation is
deferred to a future version.

\begin{figure}[htbp]
\centering
\includegraphics[width=0.58\linewidth]{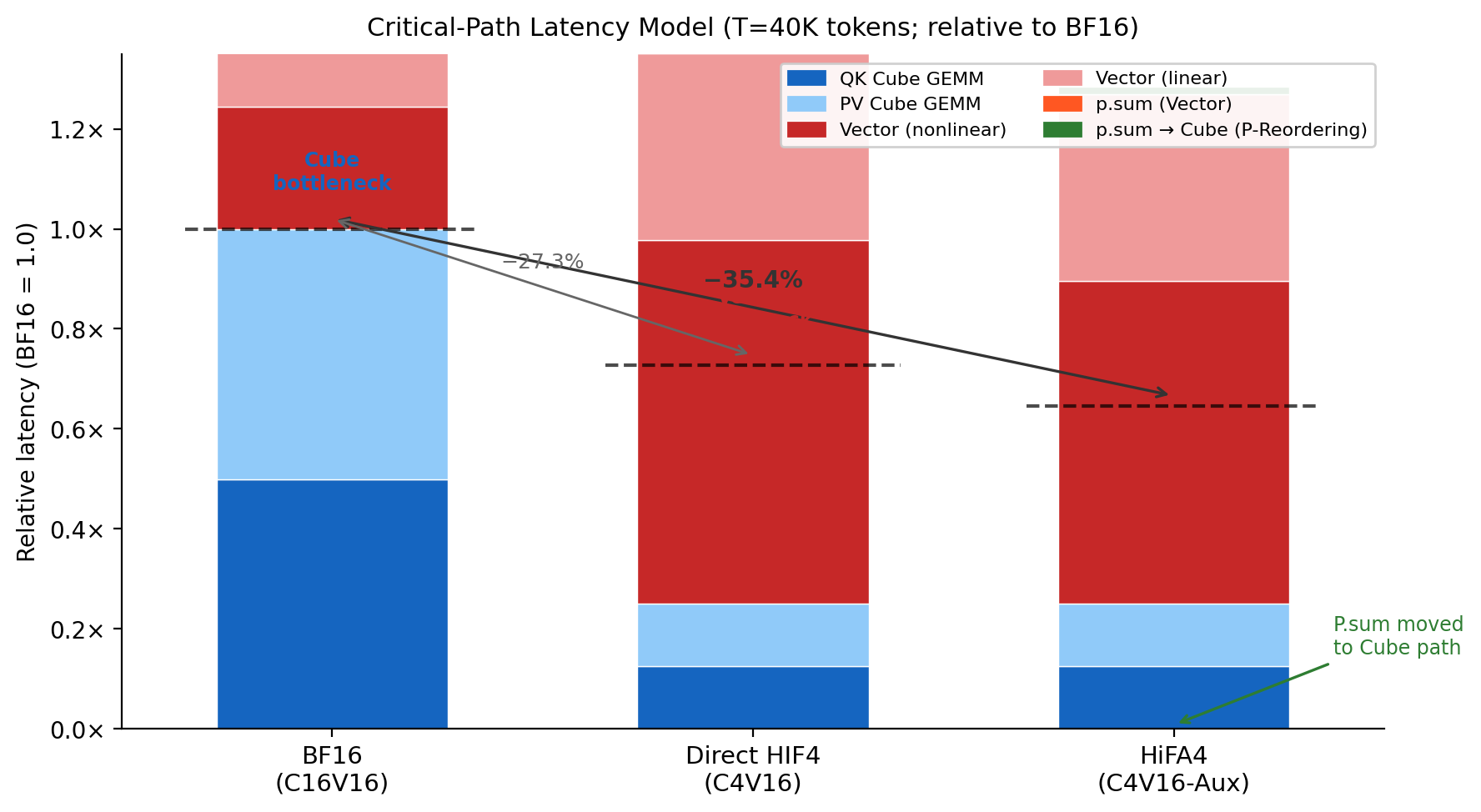}
\caption{Relative critical-path latency (BF16\,=\,1.0), from the
instruction-scheduling model described in this appendix.
HIF4 quantization reduces Cube cost but shifts the bottleneck to the
vector path.
P-Reordering moves the normalizer accumulation into the Cube GEMM,
reducing vector-path cost; the model projects a 35.4\% overall
reduction relative to BF16.
All figures are theoretical estimates; on-hardware validation is deferred.}
\label{fig:latency}
\end{figure}

\begin{table}[htbp]
\centering\small
\caption{Relative latency per stage ($T=40{,}960$ tokens;
all values normalized to BF16 = 1.0).
All figures are instruction-scheduling theoretical estimates from a
preliminary Ascend NPU scheduling analysis; on-hardware validation is deferred
(Section~\ref{sec:limitations}).}
\label{tab:latency}
\begin{tabular}{lrrrll}
\toprule
Variant & QK+PV Cube & Softmax & Other VL & Bottleneck & Relative\\
\midrule
BF16 (C16V16)     & 1.00 & 0.49 & 0.85 & Cube   & 1.00\\
Direct (C4V16)    & 0.25 & 1.46 & 0.75 & Vector & 0.73\\
\textbf{HiFA4}    & 0.25 & 1.29 & 0.75 & Vector & \textbf{0.65}\\
\bottomrule
\end{tabular}
\end{table}

The key mechanism is that 4-bit Cube GEMMs reduce the $QK^T$+$PV$
Cube cost to 25\% of BF16.
However, this shifts the critical-path bottleneck to the vector path
(softmax, rescaling, normalizer reduction).
P-Reordering fuses the normalizer accumulation into the Cube GEMM
(Section~\ref{sec:preorder}), shortening the Softmax vector-path cost
from 1.46 to 1.29 (normalized), resulting in a projected overall
reduction from 0.73 to 0.65 (35.4\% relative to BF16 = 1.0).

\end{document}